\documentclass[11pt]{article}

\usepackage[english]{babel}

\usepackage[letterpaper,top=2cm,bottom=2cm,left=3cm,right=3cm,marginparwidth=1.75cm]{geometry}

\usepackage{amsmath}
\usepackage{amsthm}
\usepackage{amsfonts}
\usepackage{graphicx}
\usepackage[authoryear]{natbib}

\usepackage[colorlinks=true, allcolors=blue]{hyperref}
\usepackage[ruled]{algorithm2e}
\usepackage[labelfont=bf,font=small]{caption}

\usepackage{tikz}
\usetikzlibrary{positioning} % For relative positioning

\usepackage{authblk} % Enables \author with \affil
\usepackage{orcidlink}

\usepackage{booktabs}
\usepackage[misc]{ifsym}
\usepackage{mwe}
\usepackage{caption}
\usepackage{amsthm}
\usepackage[T1]{fontenc}

\usepackage{pgfplots}
\pgfplotsset{compat=1.10}
\usepgfplotslibrary{fillbetween}
\usetikzlibrary{patterns}

\graphicspath{{./Figures_v2/}}

\newcommand{\iid}{\overset{\textup{iid}}{\sim}}

\usepackage{multirow}

\theoremstyle{plain}
\newtheorem{theorem}{Theorem}

\newtheorem{prop}{Proposition}
\newtheorem{propu}{Proposition}

\theoremstyle{definition}
\newtheorem{definition}{Definition}
\theoremstyle{remark}
\newtheorem{remark}{Remark} % unnumbered

\title{MCMC for Bayesian estimation of Differential Privacy from Membership Inference Attacks}
\author[1]{Ceren Yıldırım}
\author[1,2]{Kamer Kaya}
\author[1,2]{Sinan Yıldırım\,\orcidlink{0000-0001-7980-8990}}
\author[1]{Erkay Savaş}

\affil[1]{Sabancı University, Faculty of Natural Sciences and Engineering, İstanbul 34956, Turkey\par\texttt{\{cerenyildirim, kaya, sinanyildirim, erkays\}@sabanciuniv.edu}}
\affil[2]{Sabancı University, VERIM, 34956, İstanbul, Turkey}

\begin{document}
\maketitle

\begin{abstract}
% The abstract should briefly summarize the contents of the paper in 150--250 words.
We propose a new framework for Bayesian estimation of differential privacy, incorporating evidence from multiple membership inference attacks (MIA). Bayesian estimation is carried out via a Markov chain Monte Carlo (MCMC) algorithm, named \texttt{MCMC-DP-Est}, which provides an estimate of the full posterior distribution of the privacy parameter (e.g., instead of just credible intervals). Critically, the proposed method does \emph{not} assume that privacy auditing is performed with the most powerful attack on the worst-case (dataset, challenge point) pair, which is typically unrealistic. Instead, \texttt{MCMC-DP-Est} jointly estimates the strengths of MIAs used \emph{and} the privacy of the training algorithm, yielding a more cautious privacy analysis. We also present an economical way to generate measurements for the performance of an MIA that is to be used by the MCMC method to estimate privacy. We present the use of the methods with numerical examples with both artificial and real data.

\end{abstract}

\section{Introduction} \label{sec: Introduction}

Differential privacy (DP) has emerged as a gold standard for quantifying and guaranteeing privacy in data analysis and machine learning \citep{Dwork_2006, Dwork_and_Roth_2014}. DP provides a mathematical framework to limit the impact of an individual's data on the output of a random algorithm, enabling robust privacy guarantees regardless of the adversary’s knowledge other than that individual. DP is defined below.

\begin{definition}[DP] \label{defn: Differential privacy}
An algorithm $\mathcal{A}$ with input space $\bigcup_{n = 1}^{\infty} \mathcal{Z}^{n}$ and output space $\Omega$ is $(\epsilon, \delta)$-DP if for every $n \in \mathbb{N}$, $D \in \mathcal{Z}^{n}$, $z \in \mathcal{Z}$, and $E \subseteq \Omega$, we have
\begin{align*}
P(\mathcal{A}(D) \in E) \leq e^{\epsilon} P(\mathcal{A}(D \cup \{z \}) \in E) + \delta,\\
P(\mathcal{A}(D \cup \{z\}) \in E) \leq e^{\epsilon} P(\mathcal{A}(D) \in E) + \delta.
\end{align*}
\end{definition}
Quantifying the privacy of practical implementations remains a challenging task. Theoretical lower bounds for $(\epsilon, \delta)$ have been extensively analyzed for a variety of mechanisms, such as noise-adding mechanisms~(Laplace, Gaussian, etc.) \citep{Dwork_et_al_2006, Dwork_and_Roth_2014}, subsampling \citep{Balle_et_al_2018}, and their composition \citep{Kairouz_et_al_2015}. Other theoretical definitions of DP have also been used to derive lower bounds of DP \citep{Bun_and_Steinke_2016, Dong_et_al_2022, Mironov_2017}. However, theoretical lower bounds are not tight for many practical algorithms whose revealed outputs result from a series of calculations involving randomness. An example is when a training algorithm outputs \emph{only} its final model, possibly following randomizing steps like random initialization, random updates (e.g., due to subsampling and/or noisy gradients), or output perturbation. In such a case, a large gap between the theoretical bounds and the actual privacy is shown to exist \citep{Nasr_et_al_2025}. The privacy-auditing, or privacy estimation, of such complex but practical algorithms, through numerical estimation of $(\epsilon, \delta)$, has become an emerging research line \citep{Andrew_et_al_2024, Hyland_and_Tople_2022, Jagielski_et_al_2020, Lu_et_al_2022, Maddock_et_al_2023, Nasr_et_al_2023, Nasr_et_al_2021, Leemann_et_al_2023, Nasr_et_al_2025, Pillutla_et_al_2023, Steinke_et_al_2023, Zanella-Beguelin_et_al_2023}. This study follows this line by proposing a new framework for Bayesian privacy estimation.

Privacy auditing methods leverage the relation between DP and \emph{Membership Inference Attacks}~(MIA) \citep{Shokri_et_al_2017, Yeom_et_al_2018, Carlini_et_al_2022, Ye_et_al_2022} to estimate $\epsilon$ and $\delta$. This is because DP particularly guarantees the protection of sensitive input data against MIAs. While there are different types of MIAs, by Def. \ref{defn: Differential privacy}, the leave-one-out attack (L-attack in \citet{Ye_et_al_2022}) is the one directly relevant to the DP of a private algorithm. In an L-attack, a data set $D$, a point $z \in \mathcal{Z}$, and a random output of $\mathcal{A}$ are given to the attacker who uses the given information to infer whether $D$ or $D \cup \{ z \}$ was used by $\mathcal{A}$. %This is named the ``L-attack'' (for ``leave-one-out'') in \cite{Ye_et_al_2022}. 
For \emph{any} $D, z$ and \emph{any} test, the type I and type II error probabilities are lower-bounded by a curve determined by the $(\epsilon, \delta)$ of $\mathcal{A}$, see Theorem \ref{thm: privacy and errors}.

Two main issues deserve caution in empirical privacy analysis based on MIAs.
\begin{enumerate}
\item \emph{The attack strategy to audit privacy is typically not the strongest}: For a given pair $(D, z)$, the strongest attack that decides based on $\mathcal{A}$'s output is known to be the likelihood ratio test (LRT). However, practical MIAs to audit privacy merely approximate the LRT \citep{Sablayrolles_et_al_2019} with a limited computational budget, e.g.\ using metrics that are loss-based \citep{Yeom_et_al_2018, Ye_et_al_2022, Carlini_et_al_2022}, gradient-based \citep{Nasr_et_al_2023}, etc. On the other hand, an adversary can design more powerful attacks than the one used for auditing, provided computational budget. It is also difficult to analytically characterize the gap between the performance of a given MIA and the LRT. Therefore, treating a given MIA as \emph{the} strongest attack may lead to overconfident estimates about the privacy of an algorithm (`overconfident', because weaker attacks imply stronger privacy; see Remark \ref{rem: overconfidence}). This is particularly dangerous since, based on overconfident estimations, private data may be leaked to a greater extent than it is permitted.

\item \emph{The challenge base $(D, z)$ may not be the absolute ``worst-case'' challenge base}. Moreover, existing efforts such as \citet{Nasr_et_al_2021, Lu_et_al_2022, Zanella-Beguelin_et_al_2023, Pillutla_et_al_2023} do not \emph{theoretically} guarantee to find such a point. Therefore, \emph{even if} the strongest attack is applied on the challenge base $(D, z)$, its observed false positive and false negative counts should \emph{not} directly be used to upper bound $\epsilon$. This is because tests on another $(D, z)$ challenge base could result in a larger upper bound on $\epsilon$.
\end{enumerate}

% The proposed method does \emph{not} require the assumption of a strongest attack. Instead, in its most general version, it assumes that, given the privacy parameters, an attack performance is uniformly distributed in a region defined by the privacy parameters.

% When integrating results from multiple attacks, it may be important to account for the statistical dependencies among them. For example, when the same dataset is used to generate multiple attacks—such as by varying the pairs of challenge bases—their error performances may be dependent, even though the attacks themselves are distinct. Conversely, attacks performed on different datasets for the same DP algorithm can often be treated as statistically independent. Similarly, independence may be assumed when different attack strategies, such as different decision rules, are employed to infer the membership status of challenge bases. This paper introduces a flexible Bayesian framework capable of handling both dependent and independent attack scenarios, enabling a principled and probabilistic estimation of privacy parameters.

% In this paper, we present (i) a Bayesian methodology for privacy estimation that aims to circumvent the issues discussed above, and (ii) a new MIA to be used in privacy estimation. 

Considering the above issues, we propose a new Bayesian methodology for empirical privacy estimation. Our contributions are as follows.
\begin{enumerate}
\item \textbf{A new posterior sampling method for DP estimation:} We develop a joint probability distribution for the privacy parameter, attack strengths, and false negatives/positives involved in \emph{block-box} auditing a private algorithm, where one has access only to the output of $\mathcal{A}$ and not its intermediate results. Suited to this model, we propose a Markov Chain Monte Carlo (MCMC) method named \texttt{MCMC-DP-Est} (Alg.\ \ref{alg: MCMC-DP-Est}), adopted from \citet{Andrieu_et_al_2020}, for Bayesian estimation of privacy. The advantages of \texttt{MCMC-DP-Est} are as follows:
\begin{itemize}
\item \textbf{Full posterior distribution:} Beyond credible intervals, it returns the \emph{full posterior distribution} of the privacy parameters as in \citet{Zanella-Beguelin_et_al_2023}.
\item \textbf{Combining multiple results:} With the fully Bayesian treatment, the algorithm can combine the false negative and false positive counts from \emph{multiple} $(D, z)$ points (and possibly from multiple attack strategies). In particular, \emph{no tried attacks need to be thrown away}. This also enables leveraging new $(D, z)$ points or attack strategies to refine the privacy estimates coherently.
\item \textbf{Cautious treatment of attack strengths:} Related to the above discussion, the probabilistic model based on which MCMC is used in this work does \emph{not} assume that the attack used is the strongest attack possible under the privacy constraint or it is performed on the worst-case $(D, z)$ pair (though it can be adapted to include such cases). Instead, we parametrize the average strength of the applied tests/challenge bases by a parameter $s \in [0, 1]$, and use MCMC to jointly estimate both $s$ and $\epsilon$.
\end{itemize}

\item \textbf{A method for measuring MIA performance:} We present a parametric loss-based MIA, adopting LiRA \citet{Carlini_et_al_2022}, to feed privacy auditing methods (including ours) with informative error counts. We propose a computationally efficient way to measure the MIA performance. The measurements, the numbers of false positives and false negatives, are to be fed to the MCMC algorithm as observations. 

\item {\textbf{An extension of the joint probability model} that allows statistical dependencies among attacks is also discussed briefly in Remark \ref{rem: Dependent challenge bases and MIAs} in Section \ref{sec: Joint probabilistic model for privacy- and MIA-related variables} and more in detail in Appendix \ref{appndx: Modeling Dependent MIAs}. Attack results can be statistically dependent, for example, when a common $z$ point is paired with different $D$ datasets from the population.}

\end{enumerate}

Section \ref{sec: Bayesian estimation of privacy} presents the joint probability model and \texttt{MCMC-DP-Est}. Section \ref{sec: The MIA attack and measuring its performance} presents an MIA and an experiment design to collect performance measures for the attack, which are to be fed to the MCMC algorithm as observations. Section \ref{sec: Experiments} presents the experiments. Section \ref{sec: Discussion} concludes the paper.

\subsection{Related work} \label{sec: Related work}
Privacy estimation methods exploit theoretical results and approach their guarantees empirically. For example, the privacy estimation in \citet{Hyland_and_Tople_2022} estimates the privacy of stochastic gradient descent (SGD) based on the relationship between the sensitivity of the output and privacy. However, in most studies, the relation between MIAs and Definition \ref{defn: Differential privacy} of DP has been exploited. For example, \citet{Jagielski_et_al_2020} derives Clopper-Pearson confidence intervals for $\epsilon$ from MIAs carefully designed for SGD with clipping. \citet{Nasr_et_al_2021} uses Clopper-Pearson confidence intervals, too, but additionally finds the worst-case pairs $(D, z)$ to improve the bound on the privacy estimates. In contrast to frequentist estimates in \citet{Jagielski_et_al_2020, Nasr_et_al_2021}, Bayesian privacy estimation is proposed in \citet{Zanella-Beguelin_et_al_2023}, where Bayesian \emph{credible intervals} are provided for $\epsilon$. Privacy estimation has also been extended to other definitions of privacy \citep{Leemann_et_al_2023, Nasr_et_al_2023, Pillutla_et_al_2023} and to federated learning \citep{Andrew_et_al_2024, Maddock_et_al_2023}.

The quality of privacy estimation through MIAs depends heavily on the quality of MIAs, i.e., their power to distinguish membership and non-membership. Several MIAs (tests) have been proposed in the literature. If black-box privacy-auditing is performed, the loss function of the trained model is typically involved in the MIA decision rule, and using the loss function is shown to be Bayes optimal in \citet{Sablayrolles_et_al_2019}. The LOSS attack \citep{Yeom_et_al_2018} uses the loss function as its test statistic. This attack is also used as an approximation of the LRT under some conditions~\citep{Ye_et_al_2022} that correspond to the output of the training model behaving like a sample from its posterior distribution. The LOSS attack has been reported to be weak in identifying memberships (TPs) and strong in identifying non-memberships (FPs) \citep{Carlini_et_al_2022}. As an alternative, \citet{Carlini_et_al_2022} proposes LiRA, a more direct approximation of the LRT that considers the distribution of the loss function under both hypotheses.

\section{Bayesian estimation of privacy} \label{sec: Bayesian estimation of privacy}
We present the joint probability distribution for the variables regarding the privacy of $\mathcal{A}$, error probabilities of MIAs, and the observed FP and FN counts for each MIA. Then, we present the MCMC algorithm for the privacy estimation of $\mathcal{A}$ according to that joint probability distribution. But first, we provide some preliminaries and introduce some concepts for clarity.
Definition \ref{defn: challenge base} assigns a specific meaning to the term `challenge base' for the clarity of presentation.
\begin{definition}[Challenge base] \label{defn: challenge base}
A challenge base is a pair $(D, z)$, where $D \in \bigcup_{n = 1}^{n} \mathcal{Z}^{n}$ and $z \in \mathcal{Z}$ with $z \notin D$.
\end{definition}
Next, we define an MIA as a statistical test specified by a challenge base, a critical region, and a private algorithm whose output serves as the observation point for that test.

\begin{definition}[MIA] \label{defn: MIA}
An MIA is a statistical test specified by $(D, z, \mathcal{A}, \phi, \alpha, \beta)$, where $\phi: \Omega \mapsto \{0, 1\}$ is a (possibly random) decision rule for the absence or presence of $z$ in the input of $\mathcal{A}$ based on a random outcome $\theta$ from $\mathcal{A}$, 
and $\alpha, \beta$ are the type I and type II error probabilities given by
\[
\alpha = P(\phi(\theta) = 1| \theta \sim \mathcal{A}(D)), \quad \beta = P(\phi(\theta) = 0 | \theta \sim \mathcal{A}(D \cup \{ z \}) ).
\]
\end{definition}

The theorem below from \citet[Theorem 2.1]{Kairouz_et_al_2017} is central to the methodology presented in this paper. The theorem sets an upper bound on the accuracy of MIAs whose sample is the output of an $(\epsilon,\delta)$-DP algorithm.
\begin{theorem} \label{thm: privacy and errors}
$\mathcal{A}$ is $(\epsilon, \delta)$-DP if and only if, for any $D \in \mathcal{Z}^{n}$ and $z \in \mathcal{Z}$, and a decision rule $\phi$, the MIA $(D, z, \phi, \mathcal{A}, \alpha, \beta)$ satisfies $(\alpha, \beta) \in \mathcal{R}(\epsilon, \delta)$, where
\[
\mathcal{R}(\epsilon, \delta) := \left\{(x,y) \in [0, 1]^{2} : \begin{matrix} x + e^{\epsilon}y \geq 1 - \delta,\,  y + e^{\epsilon}x \geq 1 - \delta, \\  y + e^{\epsilon}x \leq e^{\epsilon} + \delta, \, x + e^{\epsilon}y \leq e^{\epsilon} + \delta \end{matrix} \right\}.
\]
See Figure \ref{fig: privacy error regions} (top left) for an illustration of $\mathcal{R}(\epsilon, \delta)$. 
\end{theorem}

\begin{remark} 
In this work, we will assume $\delta \geq 0$ fixed and known and focus on estimating the parameter $\epsilon$. However, if an algorithm is $(\epsilon_{1}, \delta)$, it is also $(\epsilon_{2}, \delta)$ for any $\epsilon_{2} > \epsilon_{1}$. To prevent ambiguity, by ``estimating privacy,'' we specifically mean estimating 
\[
\epsilon := \inf \{ \epsilon_{0} \geq 0: \mathcal{A} \text{ is } (\epsilon_{0}, \delta)\text{-DP}\}.
\]
\end{remark}

\subsection{Joint probabilistic model for privacy- and MIA-related variables} \label{sec: Joint probabilistic model for privacy- and MIA-related variables}
We present in detail the joint probability model illustrated in Figure \ref{fig: privacy error regions} (bottom).

\subsubsection{Observed error counts} 
We assume that there are $n \geq 1$ challenge bases $(D_{i}, z_{i})$, $i = 1, \ldots, n$. On each  challenge base, an MIA $(D_{i}, z_{i}, \phi_{i}, \mathcal{A}, \alpha_{i}, \beta_{i})$ is challenged $N_{i,0}, N_{i,1}$ times under $H_{0}, H_{1}$, respectively. More explicitly, for $j = 1, \ldots, N_{i,0}$, we challenge an MIA $(D_{i}, z_{i}, \phi_{i}, \mathcal{A}, \alpha_{i}, \beta_{i})$ with $\theta_{0}^{(j)}  \sim \mathcal{A}(D)$. We collect $X_{i}$, the number of false positives out of the $N_{i,0}$ challenges. Likewise, we challenge the same MIA $N_{i,1}$ times with $\theta_{1}^{(j)} \sim \mathcal{A}(D \cup \{ z \})$. We collect $Y_{i}$, the number of false negatives out of the $N_{i,1}$ challenges.

When the tests for each challenge base are independent, $X_{i}$ and $Y_{i}$ become independent binomials and their conditional distributions become
\begin{equation} \label{eq: conditional distributions for independent tests}
g(X_{i}, Y_{i} | \alpha_{i}, \beta_{i}) = \text{Binom}(X_{i} | N_{i,0}, \alpha_{i}) \times \text{Binom}(Y_{i} | N_{i,1}, \beta_{i}),
\end{equation}
where $\alpha_{i}, \beta_{i}$ are the error probabilities of the $i$'th MIA. Other distributions may arise with dependent tests, e.g.\ because of using common shadow models to learn the null and alternative hypotheses. We discuss such a case in Section \ref{sec: Measuring the performance of the MIA}.

\begin{figure}[tbh]
\begin{minipage}{0.40\textwidth}
\resizebox{\textwidth}{!}{
\begin{tikzpicture}
\pgfmathsetmacro{\dDP}{0.1}  % Set a to 2
\pgfmathsetmacro{\eDP}{1}  % Set a to 2
\begin{axis}[axis lines=middle,
            axis equal, % to keep the aspect ratio correct
            xlabel=$\alpha$,
            ylabel=$\beta$,
            enlargelimits,
            ytick={0, 0.9, 1},
            yticklabels={0, $1-\delta$, 1},
            xtick={0, 0.9, 1},
            xticklabel style={rotate=270},
            xticklabels={0,$1-\delta$, 1},
            xmin=-0, xmax=1.2, ymin=0, ymax=1.2,, scale = 0.7],
\addplot[name path=B1,blue,domain={0:(1 - \dDP)/(1 + e^\eDP)}] {(1-\dDP)-e^(\eDP)*(x)};
\addplot[name path=B2,blue,domain={((1 - \dDP)/(1 + e^\eDP)):(1 - \dDP)}] {(1 - \dDP)/(1+e^\eDP)-e^(-\eDP)*(x-(1 - \dDP)/(1+e^\eDP))};
\addplot[name path=B0,blue,domain={(1 - \dDP):1}] {0};
\addplot[name path=A0,blue,domain={0: \dDP}] {1};
\addplot[name path=A1,blue,domain={\dDP: (1 + \dDP*e^-\eDP)/(1+e^-\eDP)}] {1 - (x-\dDP)*e^(-\eDP)};
\addplot[name path=A2,blue,domain={(1 + \dDP*e^(-\eDP))/(1+e^(-\eDP)): 1}] {(1 + \dDP*e^(-\eDP))/(1+e^(-\eDP)) - e^(\eDP)*(x-(1 + \dDP*e^(-\eDP))/(1+e^(-\eDP)))};
\addplot[fill=yellow!50, draw = none]fill between[of=B1 and A0, soft clip={domain=0:1}];
\addplot[fill=yellow!50, draw = none]fill between[of=B1 and A1, soft clip={domain=0:1}];
\addplot[fill=yellow!50, draw = none]fill between[of=B2 and A2, soft clip={domain=0:1}];
\addplot[fill=yellow!50, draw = none]fill between[of=B0 and A2, soft clip={domain=0:1}];
\node[coordinate,pin=90:{$\mathcal{R}(\epsilon, \delta)$}] at (axis cs:0.6,0.4){};
\end{axis}
\end{tikzpicture}
}
\end{minipage}
\hfill
\begin{minipage}{0.60\textwidth}
\vspace*{-3ex}
\resizebox{\textwidth}{!}{
\begin{tikzpicture}
\pgfmathsetmacro{\dDP}{0.1}  % Set a to 2
\pgfmathsetmacro{\eDP}{1}  % Set a to 2
\pgfmathsetmacro{\dDPb}{0.06}  % Set a to 2
\pgfmathsetmacro{\eDPb}{0.6}  % Set a to 2
\begin{axis}[clip=false, axis lines=middle,
            axis equal, % to keep the aspect ratio correct
            xlabel=$\alpha$,
            ylabel=$\beta$,
            enlargelimits,
            ytick={0, 1},
            yticklabels={0, 1},
            xtick={0, 1},
            xticklabels={0, 1},
            xmin=-0, xmax=1.2, ymin=0, ymax=1.2, scale = 0.7]
\addplot[name path=B01,blue,domain={0:(1 - \dDP)/(1 + e^\eDP)}] {(1-\dDP)-e^(\eDP)*(x)};
\addplot[name path=B02,blue,domain={((1 - \dDP)/(1 + e^\eDP)):(1 - \dDP)}] {(1 - \dDP)/(1+e^\eDP)-e^(-\eDP)*(x-(1 - \dDP)/(1+e^\eDP))};
\addplot[name path=B00,blue,domain={(1 - \dDP):1}] {0};
\addplot[name path=A00,blue,domain={0: \dDP}] {1};
\addplot[name path=A01,blue,domain={\dDP: (1 + \dDP*e^-\eDP)/(1+e^-\eDP)}] {1 - (x-\dDP)*e^(-\eDP)};
\addplot[name path=A02,blue,domain={(1 + \dDP*e^(-\eDP))/(1+e^(-\eDP)): 1}] {(1 + \dDP*e^(-\eDP))/(1+e^(-\eDP)) - e^(\eDP)*(x-(1 + \dDP*e^(-\eDP))/(1+e^(-\eDP)))};
\addplot[name path=B1,green,domain={0:(1 - \dDPb)/(1 + e^\eDPb)}] {(1-\dDPb)-e^(\eDPb)*(x)};
\addplot[name path=B2,green,domain={((1 - \dDPb)/(1 + e^\eDPb)):(1 - \dDPb)}] {(1 - \dDPb)/(1+e^\eDPb)-e^(-\eDPb)*(x-(1 - \dDPb)/(1+e^\eDPb))};
\addplot[name path=B0,green,domain={(1 - \dDPb):1}] {0};
\addplot[name path=A0,green,domain={0: \dDPb}] {1};
\addplot[name path=A1,green,domain={\dDPb: (1 + \dDPb*e^-\eDPb)/(1+e^-\eDPb)}] {1 - (x-\dDPb)*e^(-\eDPb)};
\addplot[name path=A2,green,domain={(1 + \dDPb*e^(-\eDPb))/(1+e^(-\eDPb)): 1}] {(1 + \dDPb
*e^(-\eDPb))/(1+e^(-\eDPb)) - e^(\eDPb)*(x-(1 + \dDPb*e^(-\eDPb))/(1+e^(-\eDPb)))};
\addplot[fill=yellow!50, draw = none]fill between[of=A00 and A0, soft clip={domain=0:1}];
\addplot[fill=yellow!50, draw = none, pattern color=yellow!50]fill between[of=A01 and A1, soft clip={domain=0:1}];
\addplot[fill=yellow!50, draw = none]fill between[of=A02 and A2, soft clip={domain=0:1}];
\addplot[fill=yellow!50, draw = none]fill between[of=B0 and B00, soft clip={domain=0:1}];
\addplot[fill=yellow!50, draw = none]fill between[of=B1 and B01, soft clip={domain=0:1}];
\addplot[fill=yellow!50, draw = none]fill between[of=B2 and B02, soft clip={domain=0:1}];
\addplot[fill=yellow!50, draw = none]fill between[of=B2 and B00, soft clip={domain=0:1}];
\node[coordinate,pin={[pin distance =1 cm] -20:{\parbox{4 cm}{$\mathcal{R}(\epsilon, \delta)/\mathcal{R}(0.6 \epsilon, 0.6 \delta)$ \\ Prior domain for \\ $s$-stong tests}}}] at (axis cs:0.7,0.7){};
\node[coordinate,pin={[pin distance =2.3 cm] 16:{}}] at (axis cs:0.3,0.3){};
\end{axis}
\end{tikzpicture}
}
\end{minipage}

\vspace{-0.5cm}
\centerline{
\begin{tikzpicture}[
    ->,
    auto,
    node distance=1.5cm, minimum size=0cm,
    main node/.style={}
]
% Define the nodes
\node[main node] (DP) {$\epsilon, s$};
\node[main node] (T1) [below left=0.5cm and 0.5cm of DP] {$\alpha_{1}, \beta_{1}$};
\node[draw, rectangle] (X1) [below=0.5cm of T1] {$X_{1}, Y_{1}$};
\node[main node] (Tn_placeholder) [right=0.5cm of T1] {$\ldots$};
\node[main node] (Tn) [below right=0.5cm and 0.5cm of DP] {$\alpha_{n}, \beta_{n}$};
\node[draw, rectangle] (Xn) [below=0.5cm of Tn] {$X_{n}, Y_{n}$};
% Draw the edges
\path[every node]
    (DP) edge[] node[above] {} (T1)
    (DP) edge[] node[above] {} (Tn)
    (T1) edge[] node[above] {} (X1)
    (Tn) edge[] node[above] {} (Xn);
\end{tikzpicture}
\hfill
\begin{tikzpicture}
\pgfmathsetmacro{\dDP}{0.1}  % Set a to 2
\pgfmathsetmacro{\eDP}{1}  % Set a to 2
\pgfmathsetmacro{\dDPb}{0.03}  % Set a to 2
\pgfmathsetmacro{\eDPb}{0.3}  % Set a to 2
\begin{axis}[clip=false, axis lines=middle,
            axis equal, % to keep the aspect ratio correct
            xlabel=$\alpha$,
            ylabel=$\beta$,
            enlargelimits,
            ytick={0, 1},
            yticklabels={0, 1},
            xtick={0, 1},
            xticklabels={0, 1},
            xmin=-0, xmax=1.2, ymin=0, ymax=1.2, scale = 0.7]
\addplot[name path=B01,blue,domain={0:(1 - \dDP)/(1 + e^\eDP)}] {(1-\dDP)-e^(\eDP)*(x)};
\addplot[name path=B02,blue,domain={((1 - \dDP)/(1 + e^\eDP)):(1 - \dDP)}] {(1 - \dDP)/(1+e^\eDP)-e^(-\eDP)*(x-(1 - \dDP)/(1+e^\eDP))};
\addplot[name path=B00,blue,domain={(1 - \dDP):1}] {0};
\addplot[name path=A00,blue,domain={0: \dDP}] {1};
\addplot[name path=A01,blue,domain={\dDP: (1 + \dDP*e^-\eDP)/(1+e^-\eDP)}] {1 - (x-\dDP)*e^(-\eDP)};
\addplot[name path=A02,blue,domain={(1 + \dDP*e^(-\eDP))/(1+e^(-\eDP)): 1}] {(1 + \dDP*e^(-\eDP))/(1+e^(-\eDP)) - e^(\eDP)*(x-(1 + \dDP*e^(-\eDP))/(1+e^(-\eDP)))};
\addplot[name path=B1,green,domain={0:(1 - \dDPb)/(1 + e^\eDPb)}] {(1-\dDPb)-e^(\eDPb)*(x)};
\addplot[name path=B2,green,domain={((1 - \dDPb)/(1 + e^\eDPb)):(1 - \dDPb)}] {(1 - \dDPb)/(1+e^\eDPb)-e^(-\eDPb)*(x-(1 - \dDPb)/(1+e^\eDPb))};
\addplot[name path=B0,green,domain={(1 - \dDPb):1}] {0};
\addplot[name path=A0,green,domain={0: \dDPb}] {1};
\addplot[name path=A1,green,domain={\dDPb: (1 + \dDPb*e^-\eDPb)/(1+e^-\eDPb)}] {1 - (x-\dDPb)*e^(-\eDPb)};
\addplot[name path=A2,green,domain={(1 + \dDPb*e^(-\eDPb))/(1+e^(-\eDPb)): 1}] {(1 + \dDPb*e^(-\eDPb))/(1+e^(-\eDPb)) - e^(\eDPb)*(x-(1 + \dDPb*e^(-\eDPb))/(1+e^(-\eDPb)))};
\addplot[fill=yellow!50, draw = none]fill between[of=A00 and A0, soft clip={domain=0:1}];
\addplot[fill=yellow!50, draw = none, pattern color=yellow!50]fill between[of=A01 and A1, soft clip={domain=0:1}];
\addplot[fill=yellow!50, draw = none]fill between[of=A02 and A2, soft clip={domain=0:1}];
\addplot[fill=yellow!50, draw = none]fill between[of=B0 and B00, soft clip={domain=0:1}];
\addplot[fill=yellow!50, draw = none]fill between[of=B1 and B01, soft clip={domain=0:1}];
\addplot[fill=yellow!50, draw = none]fill between[of=B2 and B02, soft clip={domain=0:1}];
\addplot[fill=yellow!50, draw = none]fill between[of=B2 and B00, soft clip={domain=0:1}];
\node[coordinate,pin={[pin distance =0.5 cm] -30:{\parbox{4 cm}{ $\quad \quad \mathcal{R}_{s}(\epsilon, \delta)$}}}] at (axis cs:0.7,0.7){};
\node[coordinate,pin={[pin distance =2.3 cm] 16.5:{}}] at (axis cs:0.3,0.3){};
\node[draw, circle, inner sep=0pt, minimum size=3pt, color=red, fill=red] (A) at (20, 50) {};
\node[color=black, minimum size=3pt] at (22, 48) {*};
\node[draw, circle, inner sep=0pt, minimum size=3pt, color=red,fill=red] (A) at (60, 15) {};
\node[color=black, minimum size=3pt] at (55, 10) {*};
\node[draw, circle, inner sep=0pt, minimum size=3pt, color=red, fill=red] (A) at (40, 40) {};
\node[color=black, minimum size=3pt] at (37, 43) {*};
\node[draw, circle, inner sep=0pt, minimum size=3pt,color=red, fill=red] (A) at (10, 80) {};
\node[color=black, minimum size=3pt] at (12, 85) {*};
\end{axis}
\end{tikzpicture}
}

\caption{{\bf Top Left}: $\mathcal{R}(\alpha, \beta)$, the unconstrained prior domain ($s = 0$) for $\alpha, \beta$ of an MIA. {\bf Top Right}: $\mathcal{R}_{0.6}(\alpha, \beta)$, prior domain for $s = 0.6$. {\bf Bottom Left}: The dependency structure of the variables involved~(a fixed $\delta$ is assumed). {\bf Bottom Right:} Realization of the variables. $\epsilon$ and $s$ set the blue and green lines, respectively; $(\alpha_{i}, \beta_{i})$ and $(X_{i}/N_{i,0}, Y_{i}/N_{i,1})$ are shown with red and black points, resp.}
\label{fig: privacy error regions}
\end{figure}
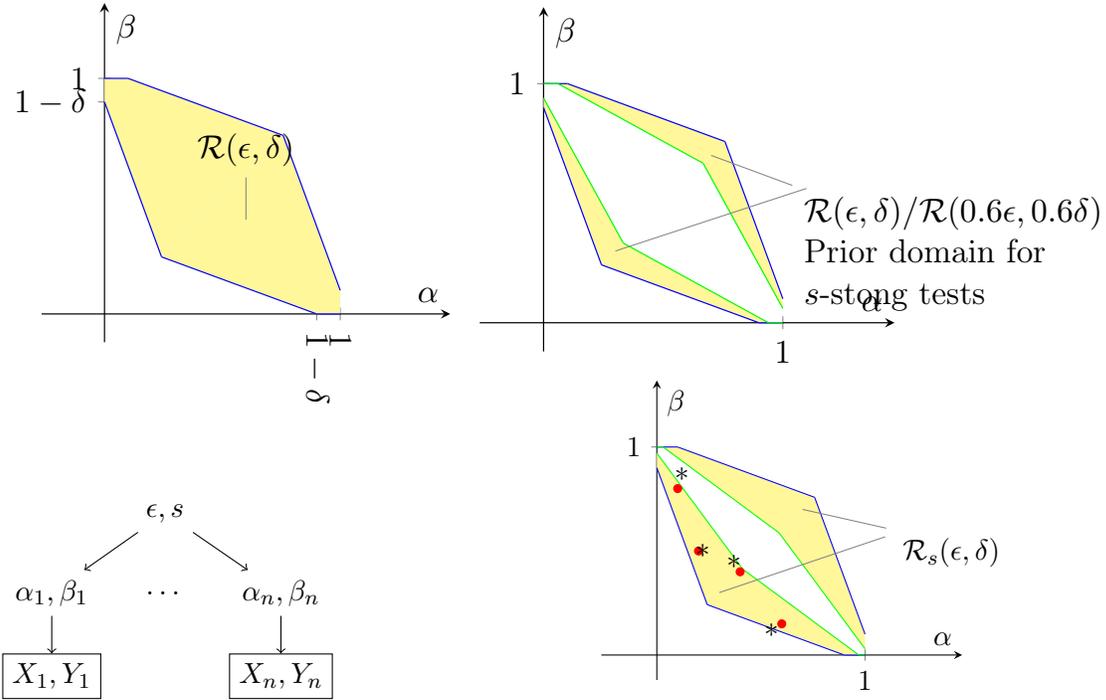

\subsubsection{True error probabilities} 
The performance of an MIA depends on $\mathcal{A}$, the challenge base $(D, z)$, as well as the decision rule $ \phi$. When we have little knowledge about the performance of a test, a convenient choice for its conditional prior distribution for $(\alpha_{i}, \beta_{i})$ given $(\epsilon, \delta)$ is the uniform distribution over $\mathcal{R}(\epsilon, \delta)$.
% , i.e.,
% \begin{equation} \label{eq: prior alpha beta}
% \alpha, \beta | \epsilon, \delta \sim \text{Unif}(\mathcal{R}(\epsilon, \delta))
% \end{equation}
% with pdf $p(\alpha, \beta | \epsilon, \delta) = \mathbb{I}[(\alpha, \beta) \in \mathcal{R}(\epsilon, \delta)] / |\mathcal{R}(\epsilon, \delta)|$, where the area is given by
% \[
% |\mathcal{R}(\epsilon, \delta)| = 1 - 2 (1 - \delta)^{2}  e^{-\epsilon}/(1 + e^{-\epsilon}).
% \]
However, uniformity over $R(\epsilon, \delta)$ may be a loose assumption for carefully designed attacks. Indeed, several works in the literature study the design of powerful MIAs by approximating the LRT  \citep{Carlini_et_al_2022, Nasr_et_al_2021, Ye_et_al_2022, Yeom_et_al_2018} or finding the worst-case (or ``best-case'' from the attacker's point of view) tuple $(D, z)$, or both. When such techniques are involved, the prior of $\alpha_{1:n}, \beta_{1:n}$ given $\epsilon, \delta$ can be modified as
\begin{equation} \label{eq: iid MIAs}
(\alpha_{1}, \beta_{1}), \ldots, (\alpha_{n}, \beta_{n}) | \epsilon, \delta \iid \text{Unif}(\mathcal{R}_{s}(\epsilon, \delta)),\quad \mathcal{R}_{s}(\epsilon, \delta) = \mathcal{R}(\epsilon, \delta) \backslash \mathcal{R}(s \epsilon, s \delta).
\end{equation}
Here, $s \in (0, 1)$ is a strength parameter for the test; the closer it is to 1, the stronger the test is expected. The parameter $s$ itself can be modeled as a random variable, for example, as $s \sim \text{Beta}(a, b)$, and it can also be estimated from the results of multiple MIAs. The pdf of $(\alpha_{i}, \beta_{i})$ in the modified case is
\begin{equation} \label{eq: prior with s density}
p_{s}(\alpha_{i}, \beta_{i} | \epsilon, \delta) = \mathbb{I}((\alpha, \beta) \in \mathcal{R}_{s}(\epsilon, \delta)) / |\mathcal{R}_{s}(\epsilon, \delta)|,
\end{equation}
where 
\[
|\mathcal{R}_{s}(\epsilon, \delta)|= 2 [ (1 - s \delta)^{2}  e^{-s\epsilon}/(1 + e^{-s\epsilon}) - (1 - \delta)^{2}  e^{-\epsilon}/(1 + e^{-\epsilon}) ]
\]
is the area of $\mathcal{R}_{s}(\epsilon, \delta)$.
Figure \ref{fig: privacy error regions} illustrates this prior for $s = 0$ (top left) and $s = 0.6$ (top right). % When multiple tests are performed (with varying challenge bases or decision rules), $s$ can be estimated from their performance results. The variable $s$ can be included in the model by assigning a prior distribution, such as $s \sim \text{Beta}(a, b)$.

\begin{remark}[The special case $s = 1$] \label{rem: overconfidence}
The choice $s = 1$ corresponds to the assumption that the strongest possible attacks are used to generate $(X_{i}, Y_{i})$ pairs. Several studies make this assumption implicitly, relying on the quality of their MIAs \citep{Jagielski_et_al_2020, Nasr_et_al_2021, Zanella-Beguelin_et_al_2023}. In many cases, the assumption is too strong since determining the worst-case challenge base $(D, z)$ \emph{and} using the most powerful test for that couple is usually intractable. As a result, taking $s = 1$ may lead to overconfident estimations about $\epsilon$. (See Figure \ref{fig: CI Example}.) More concretely, the relation between $\epsilon$ and $(\alpha, \beta)$ can be written as
\[
\{ \epsilon \leq \epsilon_{0} \} \Leftrightarrow \left\{ (\alpha, \beta) \in \mathcal{R}(\epsilon_{0}, \delta) \textit{ for any \textup{MIA}$ (D, z, \mathcal{A}, \phi, \alpha, \beta)$} \right\}.
\]
When a \emph{particular} MIA is concerned, the above gives a one-way implication as
\[
\textit{for any \textup{MIA} $(D, z, \mathcal{A}, \phi, \alpha, \beta)$}, \quad \{\epsilon \leq \epsilon_{0}\} \Rightarrow \{(\alpha, \beta) \in \mathcal{R}(\epsilon_{0}, \delta)\},
\]
which leads to 
\begin{equation} \label{eq: correct inequality of probabilities}
P(\epsilon \leq \epsilon_{0}) \leq P [(\alpha, \beta) \in \mathcal{R}(\epsilon_{0}, \delta) ].
\end{equation}
Replacing the inequality in \eqref{eq: correct inequality of probabilities} with equality is equivalent to taking $s = 1$, which would be valid only when the LRT is applied exactly \emph{and} on the worst-case challenge base $(D, z)$, which is typically not guaranteed in practice. In the absence of that strong condition, $s = 1$ leads to early saturation of the cdf $P(\epsilon \leq \epsilon_{0})$ vs $\epsilon_{0}$ and results in overconfident (credible) intervals for $\epsilon_{0}$. Section \ref{sec: Privacy Estimation with artificial test performance results} numerically demonstrates the effect of $s$ in privacy estimation.
% Several works \cite{Zanella-Beguelin_et_al_2023} uses the following equality to have a probabilistic upper bound on $\epsilon$ for a fixed $\delta \in [0, 1]$ in terms of an $(\alpha, \beta)$ point of the used attack
% \begin{equation} \label{eq: wrong implication} 
% P(\epsilon \leq \epsilon_{0}) = P [(\alpha, \beta) \in \mathcal{R}(\epsilon_{0}, \delta) ].
% \end{equation}
% This equation is implicitly based on the claim that
% \begin{equation} \label{eq: wrong implication}
% \textit{for any $T \in \mathcal{A}(\alpha, \delta)$}  \quad \epsilon^{\ast} \leq \epsilon \Leftrightarrow (\alpha, \beta) \in \mathcal{R}(\epsilon, \delta) 
% \end{equation}
\end{remark}

\begin{remark}[Dependent challenge bases and MIAs] \label{rem: Dependent challenge bases and MIAs}
In \eqref{eq: iid MIAs}, the MIA performances are assumed conditionally independent given $\epsilon, \delta$. Statistical dependency can exist among the MIAs depending on how they are designed. Dependency can occur, for example, when a group of MIAs have distinct $D$s but the \emph{same} $z$ point, or they have a common challenge base $(D, z)$ but differ in their decision rules. Dependent MIAs can also be incorporated into the statistical model. Appendix \ref{appndx: Modeling Dependent MIAs} contains a modeling approach for dependent MIAs.
\end{remark}

\subsubsection{Priors for privacy and attack strengths} 
We assume a fixed $\delta$ and estimate $\epsilon$ (and $s$). For the priors of $\epsilon$ and $s$, we consider a one-sided normal distribution $\epsilon \sim \mathcal{N}_{[0, \infty)}(0, \sigma_{\epsilon}^{2})$ and $s \sim \text{Beta}(a, b)$ independently, with $\sigma_{\epsilon}^{2} > 0$ and $a, b > 0$.

\subsubsection{Joint probability distribution} 
 Finally, the overall joint probability distribution  of $\epsilon$, $s$, $\alpha_{1:n}$, $\beta_{1:n}$, $X_{1:n}$, $Y_{1:n}$ can be written as
\begin{equation} \label{eq: joint probability model}
p_{\delta}(\epsilon, s, \alpha_{1:n}, \beta_{1:n}, X_{1:n}, Y_{1:n}) = p(\epsilon)  p(s) \prod_{i = 1}^{n} p_{\delta}(\alpha_{i}, \beta_{i} | \epsilon, s) g(X_{i}, Y_{i} | \alpha_{i}, \beta_{i}).
\end{equation}
Figure \ref{fig: privacy error regions} shows the hierarchical structure according to \eqref{eq: joint probability model}  (bottom left) and an example realization of the variables in the model (bottom right).

\subsection{Estimating privacy via MCMC} \label{sec: Estimating privacy via MCMC}

\begin{algorithm}[t]
\caption{\texttt{MCMC-DP-Est}: posterior sampling for $(\epsilon, s)$}
\label{alg: MCMC-DP-Est}
\KwIn{$M$: Number of MCMC iterations, $(\epsilon^{(0)}, s^{(0)})$: Initial values; $X_{1:n}, N_{0, 1:n}, Y_{1:n}, N_{1, 1:n}$: FP and FN counts for each challenge base and numbers of challenges under $H_{0}, H_{1}$ for each challenge base; $\sigma^{2}_{q, \epsilon} \sigma^{2}_{s, \epsilon}$: Proposal variances for $\log \epsilon$ and $s$; $K$: Number of auxiliary variables in one iteration of MCMC.}
\KwOut{Samples $\epsilon^{(i)}, s^{(i)}$, $i = 1, \ldots, M$ from $p_{\delta}(\epsilon, s, \alpha_{1:n}, \beta_{1:n} | X_{1:n}, Y_{1:n})$}

\For{$i = 1:M$}{
Draw the proposal $\epsilon' \sim \log \mathcal{N}(\log \epsilon, \sigma_{q, \epsilon}^{2})$ and $s' \sim \mathcal{N}(s, \sigma_{q,s}^{2})$.

\For{$j = 1:n$}{
Set $(\alpha_{j}^{(1)}, \beta_{j}^{(1)}) = (\alpha_{j}, \beta_{j})$.

Sample $\alpha_{j}^{(k)}, \beta_{j}^{(k)} \overset{\text{iid}}{\sim} \text{Unif}(0, 1)$ for $k = 2, \ldots, K$.

Calculate the weights
\begin{align*}
w_{j}^{(k)} &= p_{\delta}(\alpha_{j}^{(k)}, \beta_{j}^{(k)} | \epsilon, s) g(X_{j}, Y_{j} | \alpha_{j}^{(k)}, \beta_{j}^{(k)}), \quad k = 1, \ldots, K\\
w_{j}^{\prime(k)} &=  p_{\delta}(\alpha_{j}^{(k)}, \beta_{j}^{(k)} | \epsilon', s') g(X_{j}, Y_{j} |  \alpha_{j}^{(k)}, \beta_{j}^{(k)}) \quad k = 1, \ldots, K
\end{align*}
}

Acceptance probability: $A = \min \left\{1, \frac{p(s') p(\epsilon') \epsilon'}{p(s) p(\epsilon) \epsilon}\prod_{j = 1}^{n}\frac{ \sum_{k = 1}^{K} w_{j}^{\prime(k)}}{\sum_{k = 1}^{K} w_{j}^{(k)}} \right\}$

\textbf{Accept/Reject}: Draw $u \sim \text{Unif}(0, 1)$.

\uIf{$u \leq A$}{
Set $\epsilon = \epsilon', s = s'$, and $\bar{w}_{1:n}^{(1:K)} = w_{1:n}^{\prime (1:K)}$.
}\Else{
Keep $\epsilon, s$ and set $\bar{w}_{1:n}^{(1:K)} = w_{1:n}^{(1:K)}$.
}
\For{$j = 1, \ldots, n$}{
Sample $k \in \{1, \ldots, K\}$ w.p.\ $\propto \bar{w}_{j}^{(k)}$ and set $(\alpha_{j}, \beta_{j}) = (\alpha_{j}^{(k)}, \beta_{j}^{(k)})$.
}
Store $\epsilon^{(i)} = \epsilon, s^{(i)} = s$.
}

\end{algorithm}

Alg.\ \ref{alg: MCMC-DP-Est} shows the MCMC method for the \emph{joint posterior distribution}
\begin{equation} \label{eq: joint posterior of epsilon and s}
p_{\delta}(\epsilon, s, \alpha_{1:n}, \beta_{1:n} | X_{1:n}, Y_{1:n}) \propto p_{\delta}(\epsilon, s, \alpha_{1:n}, \beta_{1:n}, X_{1:n}, Y_{1:n}).
\end{equation}
We call this method \texttt{MCMC-DP-Est}. Although \texttt{MCMC-DP-Est} outputs iterates for $(\epsilon, s, \alpha_{1:n}, \beta_{1:n})$, the $\epsilon$ (or $(\epsilon, s)$)-component of the samples can be used to estimate the marginal posterior of $\epsilon$ (or $(\epsilon, s)$). \texttt{MCMC-DP-Est} is a variant of the MHAAR (Metropolis-Hastings with averaged acceptance ratios) methodology in \citet{Andrieu_et_al_2020} developed for latent variable models. (Here the latent variables are the $(\alpha_{1:n}, \beta_{1:n})$.) \texttt{MCMC-DP-Est} has $\mathcal{O}(K n)$ complexity per iteration. We state the correctness of \texttt{MCMC-DP-Est} in the following proposition. A proof is given in Appendix \ref{appndx: Proof of correctness of} and contains a strong allusion to \citet{Andrieu_et_al_2020}. 

\begin{prop} \label{prop: correctness}
For any $K > 1$, $\sigma^{2}_{q, \epsilon}$, and $\sigma^{2}_{q, s}$, \texttt{MCMC-DP-Est} in Alg.\ \ref{alg: MCMC-DP-Est} targets exactly the posterior distribution in \eqref{eq: joint posterior of epsilon and s}, in the sense that it simulates an ergodic Markov Chain whose invariant distribution is \eqref{eq: joint posterior of epsilon and s}.
\end{prop}

\section{The MIA attack and measuring its performance} \label{sec: The MIA attack and measuring its performance}
In this section, we describe the MIA used in our experiments and equip it with an experimental design to measure its performance computationally efficiently. 

\subsection{The MIA design} \label{sec: Designing the attack}

\begin{algorithm}
\caption{$b = \texttt{MIA}(\theta, D, z, \mathcal{A}, M_{0}, M_{1}, \alpha^{\ast}) $}
\label{alg: Attack}
% \KwIn{$\theta$, $D$, $x$, target false positive rate: $\alpha$, sample sizes to learn the hypotheses: $M_{0}$, $M_{1}$, private algorithm $\mathcal{A}$}
% \KwOut{Decision: $b$}

\For{$j = 1, \ldots, M_{0}$}{
Obtain $\theta_{0}^{(j)} \sim \mathcal{A}(D)$, calculate $\ell_{0}^{(j)} = L(z, \theta_{0}^{(j)})$
}
\For{$j = 1, \ldots, M_{1}$}{
Obtain $\theta_{1}^{(j)} \sim \mathcal{A}(D \cup \{z\})$, calculate $\ell_{1}^{(j)} = L(z, \theta_{1}^{(j)})$.

}

\Return $b = \texttt{LearnAndDecide}(D, z, \theta, \alpha^{\ast}, \{ \ell_{0}^{(j)} \}_{j= 1}^{M_{0}}, \{ \ell_{1}^{(j)} \}_{j= 1}^{M_{1}} )$
\end{algorithm}

The test statistic of LRT, the most powerful test, is the ratio of likelihoods $p_{\mathcal{A}}(\theta | D)/p_{\mathcal{A}}(\theta | D \cup \{ z \})$. However, the likelihoods are usually intractable due to $\mathcal{A}$'s complex structure; therefore, approximations are sought. Loss-based attacks are a common way of approximating the LRT \citep{Yeom_et_al_2018, Sablayrolles_et_al_2019, Ye_et_al_2022, Carlini_et_al_2022}. In particular, we consider a parametric version LiRA \citep{Carlini_et_al_2022}, a loss-based attack that uses the ratio $p_{L}(\ell^{\ast} | H_{0})/p_{L}(\ell^{\ast} | H_{1})$ evaluated at $\ell^{\ast} = L(z, \theta)$. The outline of a loss-based MIA is given in Alg.\ \ref{alg: Attack}. The densities $p_{L}(\cdot | H_{i})$ can be approximated via $M_{i} > 1$ shadow models $\theta_{i}^{(1)}, \ldots, \theta_{i}^{(N)}$ generated under $H_{i}$ and fitting a distribution $p_{L}(\cdot | H_{i})$ to the losses $L(z, \theta_{i}^{(1)}), \ldots L(z, \theta_{i}^{(M_{i})})$. Finally, the critical region to choose $H_{1}$ is set $\{ p_{L}(\ell | H_{0})/p_{L}(\ell | H_{1}) < \tau \}$ and $\tau$ is adjusted to have a desired target type I error probability $\alpha^{\ast}$.

\paragraph{Learning $H_{0}$ and $H_{1}$ and deciding:} Alg.\ \ref{alg: Learn and Decide} describes how we learn the distributions under both hypotheses and apply a decision. Firstly, for each $i = 0, 1$ we fit a normal distribution $\mathcal{N}(\mu_{i}, \sigma_{i}^{2})$ using the sample $\ell_{i}^{(1:M_{i})}$, where $\ell_{i}^{(j)} = L(z, \theta_{0}^{(j)})$. Then, LRT is applied to decide between $H_{0}: \ell \sim \mathcal{N}(\mu_{0}, \sigma_{0}^{2})$ and $H_{1}:\ell \sim \mathcal{N}(\mu_{1}, \sigma_{1}^{2})$ with a target type I error probability of $\alpha^{\ast}$. A derivation of the LRT is given in Appendix~\ref{appndx: Most powerful test for comparing two normal distributions with a single observation}.

\begin{algorithm}[H]
\caption{$\texttt{LearnAndDecide}(D, z, \theta, \alpha^{\ast},  \{ \ell_{0}^{(j)} \}_{j= 1}^{M_{0}}, \{ \ell_{1}^{(j)} \}_{j= 1}^{M_{1}})$}
\label{alg: Learn and Decide}

\For(\tcp*[f]{\textbf{Learn $H_{0}$ and $H_{1}$}}){$i = 0, 1$}{
Fit normal distributions for $H_{i}$ as

$\mu_{i} = \frac{1}{M_{i}} \sum_{j = 1}^{M_{i}} \ell_{i}^{(j)}, \quad \sigma_{i}^{2} = \frac{1}{M_{i}-1} \sum_{j = 1}^{M_{i}} (\ell_{i}^{(j)} - \mu_{i})^{2}$
% & (, \sigma_{i}^{2}) = \texttt{FitNormal}(\ell_{i}^{(1:M_{i})}) 
% where $\texttt{FitNormal}(x_{1}, \ldots, x_{n}) = (\bar{x}, \frac{1}{n-1} \sum_{i = 1}^{n} (x_{i} - \bar{x})^{2})$.
}

%\centerline{\textbf{Decision}}
Calculate $\ell^{\ast} = L(z, \theta)$.\tcp*[f]{\textbf{Compute  $\ell^{\ast}$, $R$ and the decision}}

% \KwIn{$\mu_{0}, \mu_{1}, \sigma_{0}, \sigma_{1}$, $\alpha$, $y$}
% \KwOut{Decision $b$}
Calculate $R = \frac{\mu_{0}/\sigma_{0}^{2} - \mu_{1}/\sigma_{1}^{2}}{1/\sigma_{0}^{2} - 1/\sigma_{1}^{2}}$ and $\delta = \frac{\mu_{0}}{\sigma_{0}} + \frac{1}{\sigma_{0}} R$ 

\Return Decision
\[
b = \begin{cases} 1 & \text{ if } \left(\ell^{\ast} + R \right)^{2} \leq \sigma_{0}^{2} F_{1, \delta^{2}}^{-1}(\alpha^{\ast}) \text{ and } \sigma_{0}^{2} > \sigma_{1}^{2}, \\ 1 & 
\text{ if } \left(\ell^{\ast} + R \right)^{2} \geq \sigma_{0}^{2} F_{1, \delta^{2}}^{-1}(1 - \alpha^{\ast}) \text{ and }\sigma_{0}^{2} < \sigma_{1}^{2}, \\ 
 0 & \text{otherwise}.\end{cases}
\]
where $F^{-1}_{d, \delta^{2}}(u)$ is the inverse cdf of $\chi^{2}_{d, \delta^{2}}$, the non-central $\chi^{2}$ dist.\ with noncentrality parameter $\delta^{2}$ and degrees of freedom $d$, evaluated at $u$.%  (Python: \texttt{scipy.stats.ncx2})
\end{algorithm}

\subsection{Measuring the performance of the MIA} \label{sec: Measuring the performance of the MIA}
% Earlier, we described how to perform an attack given $D, z$ and $\theta$. 
When the primary goal of using MIA is to audit the privacy of an algorithm, one needs to perform the attack multiple times to estimate its type I and type II error probabilities. A direct way to do this is Alg.\ \ref{alg: data generation}, where the MIA is simply run $N_{0}$ and $N_{1}$ times, each with an independent output $\theta$ and independent sets of $M_{0}, M_{1}$ shadow models for $H_{0}, H_{1}$. The cost of this procedure is proportional to $(N_{0} + N_{1}) (M_{0} + M_{1})$, which can be prohibitive. 

\begin{algorithm}
\caption{$\texttt{MeasureMIA}(\mathcal{A}, D, z, N_{0}, N_{1}, M_{0}, M_{1})$}
\label{alg: data generation}
% \KwIn{Private algorithm $\mathcal{A}$; Dataset $D$; challenge base $x$; Number of tests, $N_{0}$, $N_{1}$, Number of shadow models, $M_{0}$, $M_{1}$.}
% \KwOut{Number of true/false decisions $X$, $Y$}

Set $D_{0} = D \backslash \{ z \}$ and $D_{1} = D \cup \{ z \}$.

\For{$i = 0, 1$}{
\For{$j = 1, \ldots, N_{i}$}{
Train  $D_{i}$ and output $\theta \sim \mathcal{A}(D_{i})$.

Decide according to $\hat{d}_{i}^{(j)} = \texttt{MIA}(\theta,  D, z, \mathcal{A}, M_{0}, M_{1})$. 
}
}

\Return{$X = \sum_{j = 1}^{N_{0}} d_{0}^{(j)}$, $Y = \sum_{j = 1}^{N_{1}} 1 - d_{1}^{(j)}$}
\end{algorithm}

We present a cheaper alternative to Alg.\ \ref{alg: data generation}, in which $N$ models are trained from $H_{0}$ and $H_{1}$ and \emph{cross-feed} each other as shadow models. For each $i = 0, 1$ and $j =1, \ldots, N$, the triple $(D, z, \theta_{i}^{(j)})$ is taken as the input of MIA and the rest $\{ \theta_{i}^{(1:N)-j}, \theta_{1-i}^{(1:N)} \}$ are used as the shadow models. This is presented in Alg. \ref{alg: Measure MIA - fast}.

Although the decisions obtained with Alg. \ref{alg: data generation} are independent (given the true $\alpha, \beta$ of the MIA), those obtained with Alg. \ref{alg: Measure MIA - fast} are \emph{not} independent; they are correlated due to using the same set of shadow models. As a result, the Binomial distributions \eqref{eq: conditional distributions for independent tests} no longer hold for $X_{i}, Y_{i}$ pairs obtained from Alg. \ref{alg: Measure MIA - fast}. One can incorporate that into the joint distribution by taking the conditional distributions of $X_{i}$, $Y_{i}$ given $\alpha, \beta$ as correlated Binomial distributions \citep{Kupper_and_Haseman_1978}. An alternative, which is pursued here, is to use a \emph{bivariate normal approximation} for $(X_{i}, Y_{i})$ as
\begin{equation} \label{eq: bivariate distribution}
\mathcal{N}\left( \begin{bmatrix} N \alpha_{i} \\ N \beta_{i} \end{bmatrix}, \begin{bmatrix} \alpha_{i} (1 - \alpha_{i}) (N+N(N-1)\tau) & N^{2} \rho \sqrt{\alpha_{i}(1 - \alpha_{i}) \beta_{i}(1 - \beta_{i})} \\ N^{2} \rho \sqrt{\alpha_{i}(1 - \alpha_{i}) \beta_{i}(1 - \beta_{i})} &  \beta_{i} (1 - \beta_{i}) (N+N(N-1)\tau) \end{bmatrix} \right).
\end{equation}
The parameters $\tau, \rho$ can be estimated jointly with $\epsilon, s$ by slightly modifying \texttt{MCMC-DP-Est}. The details of this extension are given in Appendix \ref{appndx: Extension of the joint probability model for correlated error counts}.

\begin{algorithm}[H]
\caption{$\texttt{MeasureMIAFast}(\mathcal{A}, D, z, N, \alpha^{\ast})$}
\label{alg: Measure MIA - fast}
% \KwIn{$\mathcal{A}$: Private algorithm, challenge base $(D, z)$; Numbers of times the pair of hypotheses are tested, $N_{0} = N_{1} = N$, $\alpha^{\ast}$: target type I error probability}
% \KwOut{False positives and False negatives $X$, $Y$}

Set $D_{0} = D$ and $D_{1} = D \cup \{ z \}$.

%\centerline{\textbf{$N$ correlated attacks for $D_{0}$ vs $D_{1}$}}
\For(\tcp*[f]{\textbf{$N$ correlated attacks for $D_{0}$ vs $D_{1}$}}){$i = 0, 1$}{
\For{$j = 1, \ldots, N$}{
Obtain $\theta_{i}^{(j)} \sim \mathcal{A}(D_{i})$ and  calculate the loss $\ell_{i}^{(j)} = L(z, \theta_{i}^{(j)})$ 
}
}

%\centerline{\textbf{Decisions}}
\For(\tcp*[f]{\textbf{Decisions}}){$j = 1, \ldots, N$}{
% Decide based on $(\theta_{0}^{(j)}, D, z)$ and $(\theta_{1}^{(j)}, D, z)$, according to
\begin{align*}
d_{0}^{(j)} &= \texttt{LearnAndDecide}(D, z, \theta_{0}^{(j)}, \alpha^{\ast}, \{ \ell_{0}^{(i)} \}_{i= 1, i \neq j}^{N}, \{ \ell_{1}^{(i)} \}_{i = 1}^{N}) \\
d_{1}^{(j)} &= \texttt{LearnAndDecide}(D, z, \theta_{1}^{(j)}, \alpha^{\ast}, \{ \ell_{0}^{(i)} \}_{i= 1}^{N}, \{ \ell_{1}^{(i)} \}_{i = 1, i \neq j}^{N})
\end{align*}
}

\Return{$X = \sum_{j = 1}^{N} d_{0}^{(j)}$, $Y = \sum_{j = 1}^{N} 1 - d_{1}^{(j)}$}
\end{algorithm}

% \subsection{Related work}

\section{Experiments} \label{sec: Experiments}
The code to replicate all the experiments in the section can be downloaded at \url{https://github.com/cerenyildirim/MCMC_for_Bayesian_estimation}.
\subsection{Privacy estimation with artificial test performance results} \label{sec: Privacy Estimation with artificial test performance results}

\subsubsection{Role of $s$ in privacy estimation}
This experiment is designed to show the effect of the prior specification for the attack strength. For simplicity, we took $n = 1$ and focused on $N_{0,1} = N_{1,1} = N > 1$ instances of a single attack. Also, we set $X = 0.4\times N$ and $Y = 0.4 \times N$ to imitate an attack with $\alpha = \beta = 0.4$. We ran \texttt{MCMC-DP-Est} in Alg.\ \ref{alg: MCMC-DP-Est} with varying values of $s$ that are seen on the $x$-axis of the left plot in Figure \ref{fig: CI Example}. The conditional distribution $g(X_{i}, Y_{i} | \alpha_{i}, \beta_{i})$ is set to \eqref{eq: conditional distributions for independent tests}. The $90\%$ credible interval (CI) for $\epsilon$ for each run (different $s$) by computing the $5\%$- and $95\%$- empirical quartiles obtained from the last $10^{6}$ samples of the MCMC algorithm (discarding the first $10^{5}$ samples). A dramatic change is visible in the CI width as a function of $s$. CIs as narrow as those reported in \citet{Zanella-Beguelin_et_al_2023} with the same observations are obtained when $s \geq 0.9$. However, CIs are significantly wider for smaller (and arguably more realistic) values of $s$. Those results indicate the critical role of $s$, hence the importance of its estimation when it is unknown.

\begin{figure}
\centerline{
\includegraphics[scale = 0.33]{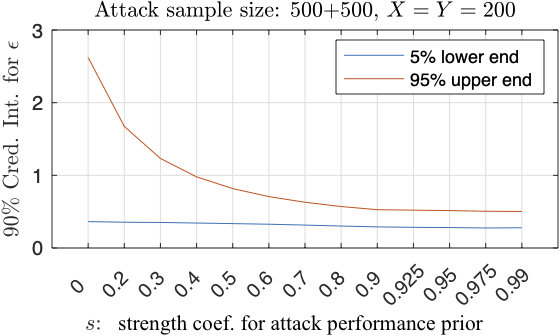}
\includegraphics[scale = 0.7]{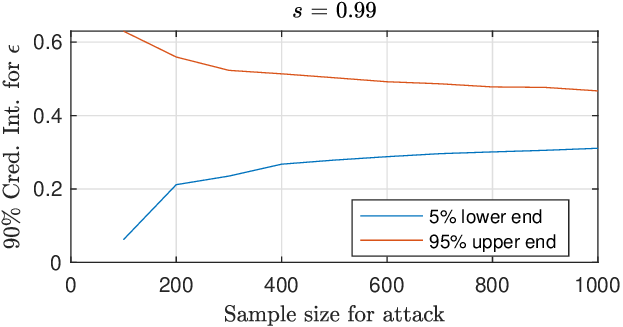}
}
\caption{{\bf Left:} $90\%$ CI for $\epsilon$ vs $s$. {\bf Right:} $90 \%$ CI for $\epsilon$ vs $N$.}
\label{fig: CI Example}
\end{figure}

\subsubsection{Estimating $\epsilon$ and $s$}
Here we show how \texttt{MCMC-DP-Est} estimates $\epsilon$ and $s$ jointly from multiple attack results in different scenarios. We took $n = 10$ and $N_{0,i} = N_{1,i} = 1000$ for all $i = 1, \ldots, n$. We considered two scenarios. % $\{X_{i}, Y_{i}, i = 1, \ldots, n\}$ as follows. 
\begin{itemize}
\item In the first scenario, we made the test strengths evenly spread over $\mathcal{R}(\epsilon, \delta)$ by generating $\alpha_{i}, \beta_{i} \iid \text{Beta}(10, 10)$, for $i = 1, \ldots, n$. The counts $X_{i}, Y_{i}$ were drawn as $X_{i} \iid \text{Binom}(N_{0,1}, \alpha_{i})$, $Y_{i} \iid \text{Binom}(N_{0,1}, \beta_{i})$, independently. 

\item In the second, we assumed relatively accurate attacks as \\
$
\begin{tabular}{c | c c c c c c c c c c}
% $i$ & 1 & 2 & 3 & 4 & 5 & 6 & 7 & 8 & 9 & 10 \\
$X_{1:10}$ & 40 & 50 & 60 & 100 & 100 & 110 & 120 & 200 & 200 & 200 \\
\hline
$Y_{1:10}$ & 250 & 200 & 150 & 100 & 120 & 100 & 100 & 80 & 70 & 60\\
\end{tabular}
$.
%}
% \\
\noindent 
\end{itemize}
The observed error rates $(X_{i}/N_{0,i}, Y_{i}/N_{1,i})$ are shown on the left-most plot in Fig.~\ref{fig: histograms example}. Alg.~\ref{alg: MCMC-DP-Est} was run to obtain $10^{6}$ samples from $p(\epsilon, s | X_{1:n}, Y_{1:n})$. As previously, $g(X_{i}, Y_{i} | \alpha_{i}, \beta_{i})$ is set to \eqref{eq: conditional distributions for independent tests}. The results in Fig.~\ref{fig: histograms example} indicate that our method can accurately estimate the attack strengths and $\epsilon$ together.

\begin{figure}[!h]
\centerline{
\includegraphics[scale = 0.9]{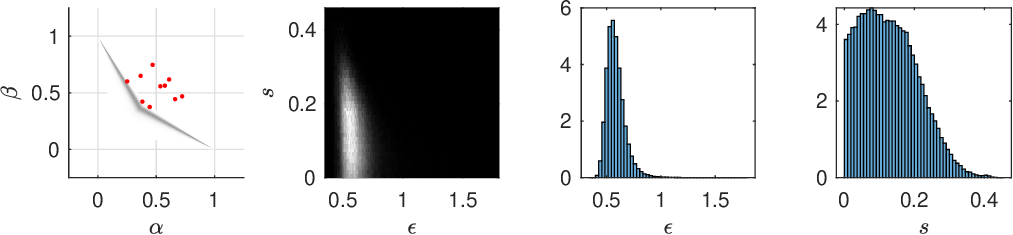}}
\centerline{\includegraphics[scale = 0.9]{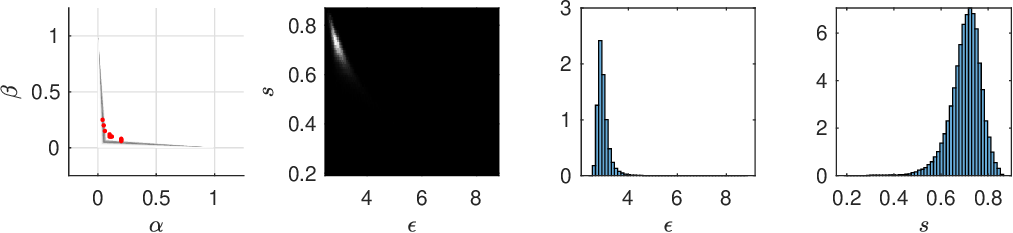}}
\caption{{\small{Posterior distributions for $\epsilon, s$ from multiple attacks. {\bf Top:} Weak attacks. {\bf Bottom:} Strong attacks. The gray area in left-most plots are ``histograms'' of $\epsilon$ for the test according to the posterior distribution of $\epsilon$ (the symmetric counterpart is omitted).}}}
\label{fig: histograms example}
\end{figure}

\subsection{Experiments with real data} \label{sec: Experiments with real data}

We considered the MNIST dataset as the population, which contains 60,000 training examples~\citep{deng2012mnist}. Each example in the set contains a $28 \times 28$-pixel image and an associated categorical label in $\{0, \ldots, 9\}$. In the experiments, we construct $D$ with a size of $999$ (to avoid too small batches while training $D \cup \{z\}$). We generated $n = 20$ challenge bases, and for each challenge base, we generated $N = 100$ challenges. Those challenges are also used as shadow models in a cross-feeding fashion, as described in Section~\ref{sec: Measuring the performance of the MIA}.

\paragraph{Training algorithms and attacks:} For $\mathcal{A}$, we considered a fully connected neural network with one hidden layer having 128 nodes and ReLU as its activation function. Meanwhile, the activation function of the output layer is softmax. We set the loss function $L(x, \theta)$ as categorical cross-entropy. To train the models, we use Keras and TensorFlow libraries~\citep{tensorflow2015-whitepaper, chollet2015keras}. For the optimizer, we use SGD with the momentum parameter $0.9$ and learning rate $0.01$.

We consider black-box auditing of four choices for $\mathcal{A}$ to audit their privacy. The algorithms differ based on the initialization and output perturbation: ($\mathcal{A}_{1}$): fixed initial, no output perturbation; ($\mathcal{A}_{2}$): random initialization, no output perturbation; ($\mathcal{A}_{3}$): fixed initial, output perturbation. ($\mathcal{A}_{4}$): random initialization, output perturbation. Output perturbation is performed by adding i.i.d.\ noise from $\mathcal{N}(0, \sigma^{2})$ to components of the trained model and releasing the noisy model. All algorithms are run for 200 epochs with a minibatch size of 100. 

\paragraph{Attack performances and privacy estimation:}
The attack performance of the MIA in Section \ref{sec: The MIA attack and measuring its performance} on the outputs of $\mathcal{A}_{1:4}$ is shown in Figure~\ref{fig: scatterplots_x_y}. The error counts are obtained with the procedure in Alg.\ \ref{alg: Measure MIA - fast} run for each algorithm. The algorithms with output perturbation used $\sigma = 0.1$. In each plot, each dot is a value $(X_{i}(\alpha^{\ast}), Y_{i}(\alpha^{\ast}))$, where $\alpha^{\ast}$ is the target type I error for the MIA. For the same challenge base $(D_i, z_i)$, several points $(X_{i}(\alpha^{\ast}), Y_{i}(\alpha^{\ast}))$ are obtained by using the $\alpha^{\ast} \in \{0.01, 0.02, \ldots , 0.99\}$, and those points are joined by a line. 

We observe that the random initialization affects the performance of the attacks visibly when output perturbation is not used ($\mathcal{A}_1$ vs $\mathcal{A}_{2}$). However, the effect of random initialization significantly drops when output perturbation is used. We also see that some challenge bases $(D_{i}, z_{i})$ allow significantly better detection than others. This is expected since the challenge bases are drawn at random. Strategies to craft worst-case scenarios \citep{Nasr_et_al_2021} can be used to eliminate those challenge bases for which the error lines are close to the $x+y =1$ line.

We turn to privacy estimation. We choose a single $(X, Y)$ point for each challenge base to feed \texttt{MCMC-DP-Est} in Alg.\ \ref{alg: MCMC-DP-Est} with $n = 20$ observations; those points are $(X_{i}(0.1), Y_{i}(0.1))$ for each $i = 1, \ldots, 20$. \texttt{MCMC-DP-Est} is run with $K = 1000$ auxiliary variables for $10^{5}$ iterations, and the first $10^{4}$ samples are discarded as burn-in. We used $\epsilon \sim \log \mathcal{N}_{[0, \infty)}(0, 10)$, one-sided normal distribution, and $s \sim \text{Unif}(0, 1) = \text{Beta}(1, 1)$ for the priors of $\epsilon$ and $s$. For $g(X_{i}, Y_{i} | \alpha_{i}, \beta_{i})$, a bivariate normal approximation in \eqref{eq: bivariate distribution} is used to account for the dependency among the decisions produced by Alg.~\ref{alg: Measure MIA - fast}. The two additional parameters $\tau, \rho$ are estimated within \texttt{MCMC-DP-Est}, as described in Appendix \ref{appndx: Extension of the joint probability model for correlated error counts}.

The 2D histograms at the bottom row of Figure~\ref{fig: scatterplots_x_y} are the posterior distributions of $(\epsilon, s)$ constructed from the samples provided by \texttt{MCMC-DP-Est}. The estimates of $(\epsilon, s)$ are as expected (e.g., randomness decreases $\epsilon$) and are consistent with the attack performances. The estimates for $s$ suggest that with more randomness in training, either the loss-based attack loses its power (relative to the best theoretical attack) or some challenge bases are no longer informative.

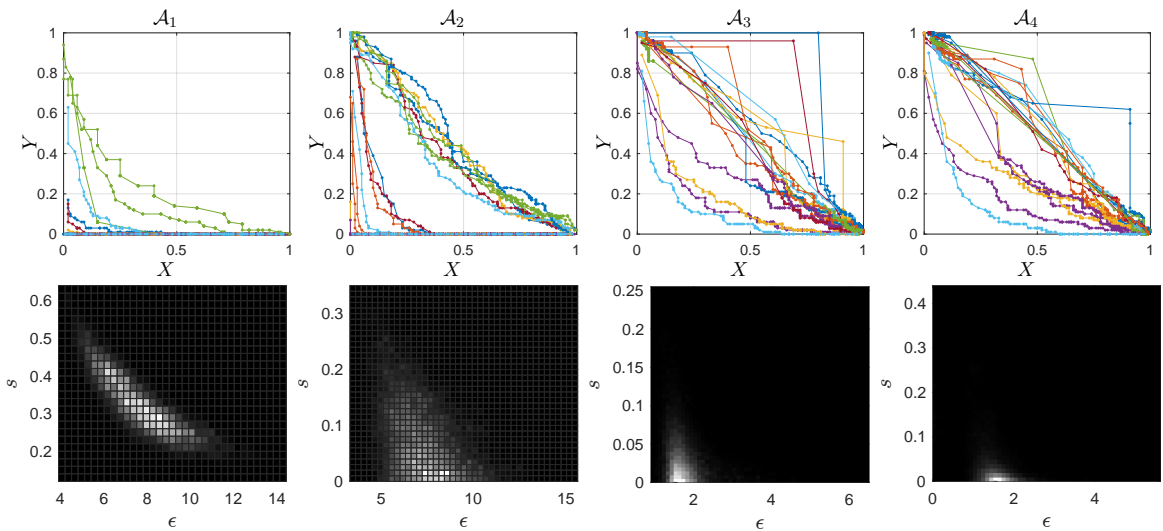
\begin{figure}[h!]
\centerline{
\begin{tikzpicture}
\node[anchor=south west,inner sep=0] (image) at (0,0) {\includegraphics[scale = 0.55]{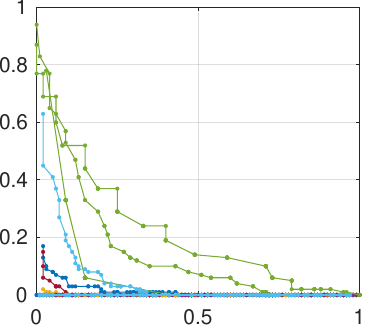} };
\begin{scope}[x={(image.south east)},y={(image.north west)}]
\node[above,scale=0.7] at (0.5,0.97) {$\mathcal{A}_{1}$};
\node[rotate=90,scale=0.7] at (-0.03,0.5) {$Y$};
\node[scale=0.7] at (0.5,-0.05) {$X$};
\end{scope}
\end{tikzpicture}

\begin{tikzpicture}
\node[anchor=south west,inner sep=0] (image) at (0,0) {\includegraphics[scale = 0.55]{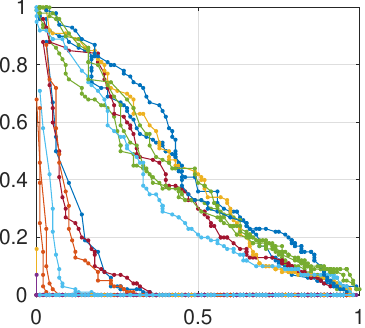} };
\begin{scope}[x={(image.south east)},y={(image.north west)}]
\node[above,scale=0.7] at (0.5,0.97) {$\mathcal{A}_{2}$};
\node[rotate=90,scale=0.7] at (-0.03,0.5) {$Y$};
\node[scale=0.7] at (0.5,-0.05) {$X$};
\end{scope}
\end{tikzpicture}

\begin{tikzpicture}
\node[anchor=south west,inner sep=0] (image) at (0,0) {\includegraphics[scale = 0.55]{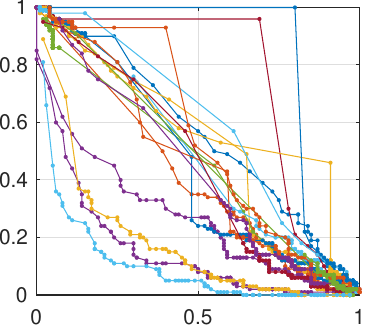} };
\begin{scope}[x={(image.south east)},y={(image.north west)}]
\node[above,scale=0.7] at (0.5,0.97) {$\mathcal{A}_{3}$};
\node[rotate=90,scale=0.7] at (-0.03,0.5) {$Y$};
\node[scale=0.7] at (0.5,-0.05) {$X$};
\end{scope}
\end{tikzpicture}

\begin{tikzpicture}
\node[anchor=south west,inner sep=0] (image) at (0,0) {\includegraphics[scale = 0.55]{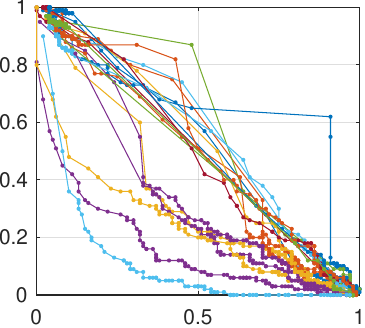} };
\begin{scope}[x={(image.south east)},y={(image.north west)}]
\node[above,scale=0.7] at (0.5,0.97) {$\mathcal{A}_{4}$};
\node[rotate=90,scale=0.7] at (-0.03,0.5) {$Y$};
\node[scale=0.7] at (0.5,-0.05) {$X$};
\end{scope}
\end{tikzpicture}
}

\centerline{
\includegraphics[scale = 0.6]{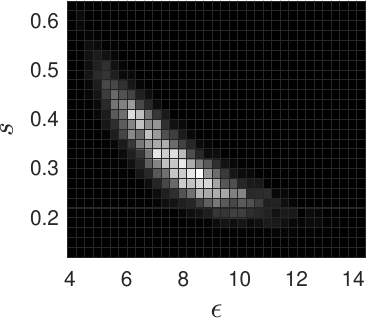} 
\includegraphics[scale = 0.6]{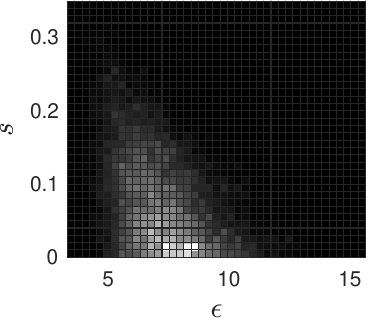} 
\includegraphics[scale = 0.6]{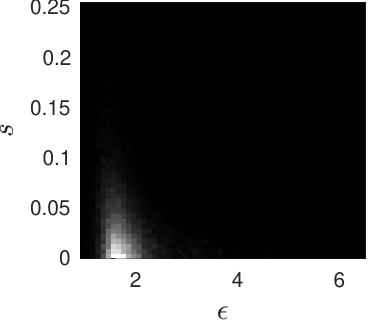} 
\includegraphics[scale = 0.6]{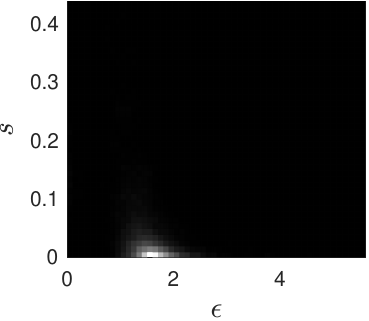} 
}
    \caption{$(X, Y)$ counts for $\mathcal{A}_{1}, \ldots, \mathcal{A}_{4}$. For output perturbation, $\sigma = 0.1$ was used.}
    \label{fig: scatterplots_x_y}
\end{figure}

We repeat the experiments for $\mathcal{A}_{4}$ with $\sigma \in \{0.01, 0.05, 0.1 \}$ for the output perturbation noise. Figure~\ref{fig: scatterplots_sigma} shows that, as expected, increasing noise makes the attacks less accurate, which in turn causes smaller estimates for $\epsilon$.

The estimates of \texttt{MCMC-DP-Est} across the audited algorithms are summarized in Table \ref{tbl: Credible intervals} with $90 \%$ CIs for $\epsilon$. Figure \ref{fig:ACFs} shows the sample autocorrelation functions (ACF) for $\epsilon$-samples of all the runs of \texttt{MCMC-DP-Est}. The fast-decaying ACFs indicate a healthy (fast-mixing) chain. For further diagnosis, we also provide the trace plots of the samples for $\epsilon$, $s$, $\tau$, $\rho$ in Appendix~\ref{appndx: Extension of the joint probability model for correlated error counts}.

\begin{figure}[h!]
    \centerline{
    \begin{tikzpicture}
    \node[anchor=south west,inner sep=0] (image) at (0,0) {\includegraphics[scale = 0.55]{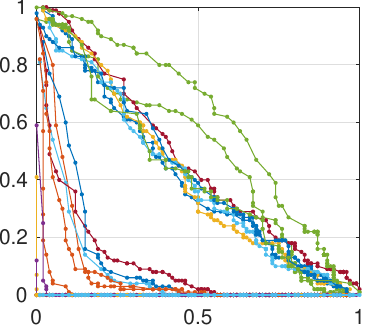} };
    \begin{scope}[x={(image.south east)},y={(image.north west)}]
            \node[above,scale=0.7] at (0.5,0.97) {$\mathcal{A}_{4}$ with $\sigma = 0.01$};
            \node[rotate=90,scale=0.7] at (-0.03,0.5) {$Y$};
            \node[scale=0.7] at (0.5,-0.05) {$X$};
        \end{scope}
    \end{tikzpicture}
    \begin{tikzpicture}
    \node[anchor=south west,inner sep=0] (image) at (0,0) {\includegraphics[scale = 0.55]{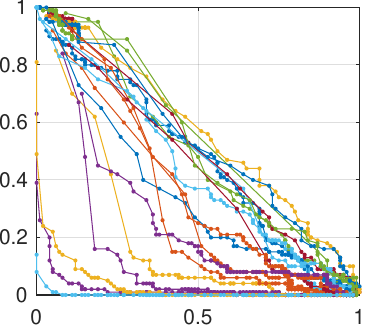} };
    \begin{scope}[x={(image.south east)},y={(image.north west)}]
            \node[above,scale=0.7] at (0.5,0.97) {$\mathcal{A}_{4}$ with $\sigma = 0.05$};
            \node[rotate=90,scale=0.7] at (-0.03,0.5) {$Y$};
            \node[scale=0.7] at (0.5,-0.05) {$X$};
        \end{scope}
    \end{tikzpicture}
    \begin{tikzpicture}
    \node[anchor=south west,inner sep=0] (image) at (0,0) {\includegraphics[scale = 0.55]{Errors_dp_attack_w_DP_random_weights_0.1.txt_1.pdf} };
    \begin{scope}[x={(image.south east)},y={(image.north west)}]
            \node[above,scale=0.7] at (0.5,0.97) {$\mathcal{A}_{4}$ with $\sigma = 0.1$};
            \node[rotate=90,scale=0.7] at (-0.03,0.5) {$Y$};
            \node[scale=0.7] at (0.5,-0.05) {$X$};
        \end{scope}
    \end{tikzpicture}
        }

    \centerline{
    \includegraphics[scale = 0.6]{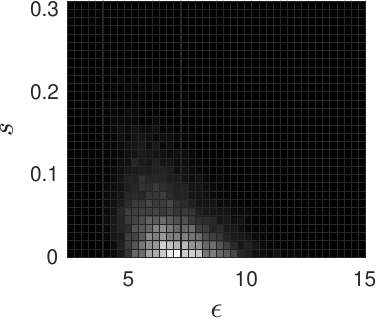} 
    \includegraphics[scale = 0.6]{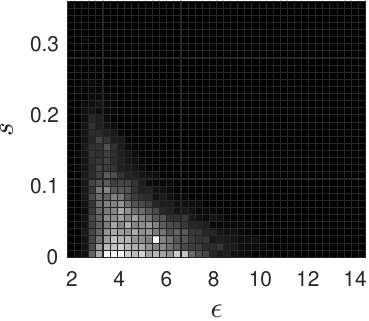} 
    \includegraphics[scale = 0.6]{histogram_dp_attack_w_DP_random_weights_0.1.txt.pdf}
    }
    \caption{Error counts and privacy estimation for \textbf{$\mathcal{A}_{4}$} with $\sigma \in [0.01, 0.05, 0.1]$}
    \label{fig: scatterplots_sigma}
\end{figure}

 \begin{table}[h]
        \centering
        \caption{$90\%$ Credible intervals for $\epsilon$}
        \label{tbl: Credible intervals}
        \begin{tabular}{c | c | c | c | c | c | c}
            \toprule
            & $\mathcal{A}_{1}$ & $\mathcal{A}_{2}$ & $\mathcal{A}_{3}$ & $\mathcal{A}_{4}$  & $\mathcal{A}_{4}$  & $\mathcal{A}_{4}$  \\
            & & & ($\sigma = 0.1$) & ($\sigma = 0.1$) & ($\sigma = 0.05$) & ($\sigma = 0.01$) \\
            % & & & \multicolumn{2}{c}{$\sigma = 0.1$} & $\sigma = 0.05$ & $\sigma = 0.01$  \\
            \hline
            Lower & 5.52  &  4.95   & 1.29 &    1.00  &  2.80 &    4.61 \\
            Upper & 10.52  &  10.00 &    2.53 &   2.12 &    7.68 &   9.62
            \end{tabular}
\end{table}

\begin{figure}
\centering
\includegraphics[width=0.5\linewidth]{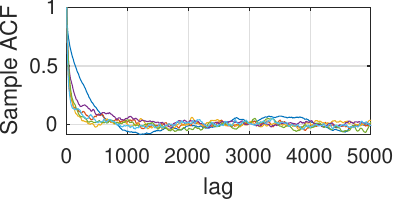} % Replace with actual image
\caption{Sample ACFs}
\label{fig:ACFs} % Label for referencing
\end{figure}
    
\section{Conclusion}\label{sec: Discussion}
In this work, we proposed a novel method for Bayesian estimation of differential privacy. Our algorithm leverages multiple $(D, z)$ pairs and attacks to refine the privacy estimates and makes no assumptions about the strength of the attacks to avoid overconfident estimations. Beyond just credible intervals, the method provides the entire posterior distribution of the privacy parameter (as well as the average attack strength). 
% Furthermore, we showed our method could be extended to cases where the attack instances are statistically dependent. 
Our experiments demonstrated that we can effectively estimate the privacy parameters of models trained under various randomness assumptions and that the resulting estimates align with attack performances. % 

We considered the ``inclusion'' versions of DP and MIAs where the pair of datasets differ by the inclusion/exclusion of a single point. The methodology similarly applies to the ``replace'' versions where dataset pairs are $D \cup \{ z \}$ and $D \cup \{z' \}$ for a pair of $z, z'$.

In the real data experiments in Section \ref{sec: Experiments with real data}, a single run of \texttt{MCMC-DP-Est} took $\approx 2.5$ minutes on Matlab on a modern laptop, which is negligibly small compared to the time needed to collect the error counts. This suggests that \texttt{MCMC-DP-Est} can feasibly be used several times to estimate $\epsilon$ at different values of $\delta$.
 
\paragraph{Limitations and future work:} Given a challenge base $(D, z)$, our statistical model considers the error counts for a single target value of type I error (i.e., a single point on each line of a plot in Figure \ref{fig: scatterplots_x_y}). As suggested by \citet{Carlini_et_al_2022}, MIA performance tests would be utilized more effectively by using false negative and false positive counts at all decision thresholds to estimate its \emph{profile function} $f(\alpha)$. Likewise, rather than 
%the classical 
$(\epsilon, \delta)$-DP, the $f$-DP \citep{Dong_et_al_2022} of $\mathcal{A}$ could be estimated as in \citet{Leemann_et_al_2023, Nasr_et_al_2023}, since $f$-DP has more complete information setting a lower bound on the profile functions of MIAs. Both extensions require prior specifications for random functions (e.g.\ \emph{a la} Gaussian process priors), which we consider an important avenue for future work.% of the current work.

\bibliographystyle{apalike}
\bibliography{references}

\appendix
\section{Proof of correctness of \texttt{MCMC-DP-Est}} \label{appndx: Proof of correctness of}
Below, we restate Proposition~\ref{prop: correctness} on the correctness of \texttt{MCMC-DP-Est} for reference.
\begin{propu}
For any $K > 1$, $\sigma^{2}_{q, \epsilon}$, and $\sigma^{2}_{q, s}$, \texttt{MCMC-DP-Est} in Alg.\ \ref{alg: MCMC-DP-Est} targets exactly the posterior distribution
\begin{equation} \label{eq: prop_correctness}
p_{\delta}(\epsilon, s, \alpha_{1:n}, \beta_{1:n} | X_{1:n}, Y_{1:n}) = p(\epsilon) p(s) \prod_{i =1}^{n} p_{\delta}(\alpha_{i}, \beta_{i} | \epsilon, s) g(X_{i}, Y_{i} | \alpha_{i}, \beta_{i})
\end{equation}
in the sense that it simulates an ergodic Markov Chain whose invariant distribution is the posterior distribution above.
\end{propu}

\begin{proof}
The posterior distribution in \eqref{eq: prop_correctness} is of the form
\[
\pi(\theta, z_{1:n}) \propto \eta(\theta) \prod_{i = 1}^{n} \gamma_{t, \theta}(z_{t})
\]
that is introduced in \citet[Section 3.1]{Andrieu_et_al_2020}, where $\theta := (\epsilon, s)$,  $z_{i} := (\alpha_{i}, \beta_{i})$, and 
\[
\gamma_{i, \theta}(z_{i}) = p_{\delta}(\alpha_{i}, \beta_{i} | \epsilon, s) g(X_{i}, Y_{i} | \alpha_{i}, \beta_{i}).
\]
Furthermore, Alg.\ \ref{alg: MCMC-DP-Est} of the paper matches exactly with \citet[Algorithm 3]{Andrieu_et_al_2020} with optional refreshment of $z$ and with $\vartheta  = (\epsilon', s')$, and
\[
q_{t, \theta, \vartheta}(z_{t}) = \mathbb{I}((\alpha, \beta) \in [0, 1]^2)
\]
(the uniform distribution over $[0, 1]^2$). \citet[Algorithm 3]{Andrieu_et_al_2020} has two \emph{move} mechanisms (labeled by $c = 1$ and $c = 2$), and at each iteration, one of them is selected at random. However, since $q_{t, \theta, \vartheta}(z_{t})$ is symmetric with respect to $\theta = (\epsilon, s), \vartheta = (\epsilon', s')$, both moves $(c = 1)$ and $(c = 2)$ in \citet[Algorithm 3]{Andrieu_et_al_2020} are identical and reduces to a single type of move as in Alg.\ \ref{alg: MCMC-DP-Est} of the paper. Finally, by \citet[Theorem 3]{Andrieu_et_al_2020}, Alg.\ \ref{alg: MCMC-DP-Est} is ergodic with invariant distribution given in \eqref{eq: prop_correctness}.
\end{proof}

\section{Extension of the joint probability model for correlated error counts} \label{appndx: Extension of the joint probability model for correlated error counts}

\subsection{Modeling dependent $(X_{i}, Y_{i})$ produced by Alg.\ \ref{alg: Measure MIA - fast}}
When the MIA on a challenge base $(D, z)$ is measured by Alg.\ \ref{alg: Measure MIA - fast}, the decisions are correlated. Let $D_{0, j}$ and $D_{1, j}$ are the $j$'th decisions when the true hypothesis  $H_{0}: \text{Data is $D$}$ and $H_{1}: \text{Data is $D \cup \{z\}$}$, respectively. Therefore, 
\[
X = \sum_{j = 1}^{N} D_{0, j}, \quad Y = \sum_{j = 1}^{N} D_{1, j}.
\]
Two sources of correlation are present.
\begin{enumerate}
\item $D_{0, i}$ and $D_{0, j}$ are correlated for all $j$; and since $D_{0, 1}, \ldots, D_{0, N}$ are interchangeable, the correlation for all $i \neq j$ is the same. We denote the common correlation by $\tau$. Assume for the sake of parsimony of the model that the common correlation is the same for $H_{0}$ and $H_{1}$ and across all the challenge bases. Then we have
\[
\text{Cov}(D_{0, i}, D_{0, j}) = \begin{cases} \alpha(1 - \alpha) & \text{ for } i = j \\ \alpha (1 - \alpha) \tau & \text{ for } i \neq j \end{cases}.
\]
Hence, the first two moments of $X$ are 
\[
\text{Var}(X) = \alpha(1 - \alpha) (N + N(N-1) \tau).
\]
Similarly, the first two moments of $Y$ are 
\[
\text{Var}(Y) = \beta(1 - \beta) (N + N(N-1) \tau).
\]
Since the variances are non-negative, we necessarily have 
\[
-\frac{1}{N-1} < \tau \leq 1.
\]

\item For every $i, j$, $D_{0, i}$ and $D_{1, j}$ are correlated and they have the same correlation (due to interchangeability), say $\rho$. Then, for any $i, j \in \{1, \ldots, N\}$,
\[
\text{Cov}(D_{0, i}, D_{1, j}) = \rho \sqrt{\alpha (1 - \alpha) \beta (1 - \beta)}.
\]
Hence,
\[
\text{Cov}(X, Y) = \sum_{i = 1}^{N} \sum_{j = 1}^{N} \text{Cov}(D_{0, i}, D_{1, j}) = N^{2} \rho \sqrt{\alpha (1 - \alpha) \beta (1 - \beta)}.
\]
Finally, the correlation between $X$ and $Y$ is
\[
\text{Corr}(X, Y) = \frac{N^{2} \rho \sqrt{\alpha (1 - \alpha) \beta (1 - \beta)}}{  \sqrt{(N + N(N-1) \tau)^{2} \alpha(1 - \alpha) \beta(1 - \beta) }} = \frac{N \rho}{1 + (N-1) \tau}.
\]
Since the correlation is between $-1$ and $1$, we necessarily have
\[
 |\rho| \leq \frac{1 + (N-1) \tau}{N}.
\]
\end{enumerate}
Combining everything, a bivariate normal approximation can be made as
\begin{equation}
\begin{bmatrix}X \\ Y \end{bmatrix} | \alpha, \beta, \tau, \rho \sim \mathcal{N}\left( \begin{bmatrix} N \alpha \\ N \beta \end{bmatrix}, \begin{bmatrix} \alpha (1 - \alpha) (N+N(N-1)\tau) & N^{2} \rho \sqrt{\alpha(1 - \alpha) \beta(1 - \beta)} \\ N^{2} \rho \sqrt{\alpha(1 - \alpha) \beta(1 - \beta)} &  \beta (1 - \beta) (N+N(N-1)\tau) \end{bmatrix} \right), \label{eq: bivariate distribution-2}
\end{equation}
where the variables $\tau$ and $\rho$ jointly satisfy $ -\frac{1}{N-1} < \tau < 1$ and  $|\rho| \leq \frac{1 + (N-1) \tau}{N}$. A suitable prior for $\tau, \rho$ is
\[
\tau \sim \mathcal{N}_{[-1/(N-1), 1]}(0, \sigma_{\tau}^{2} ), \quad \rho | \tau \sim \text{Unif}\left(-\frac{1+(N-1)\tau}{N}, \frac{1 + (N-1)\tau}{N} \right).
\]
with density
\begin{equation} \label{eq: prior for tau and rho}
p(\tau, \rho) = \begin{cases} \mathcal{N}_{[-1/(N-1), 1]}(\tau | 0, \sigma_{\tau}^{2}) \frac{N}{2 (1 + (N-1)\tau)} & \text{ for } -\frac{1}{N-1} < \tau < 1,  |\rho| \leq \frac{1 + (N-1) \tau}{N} \\ 0 & \text{ else }\end{cases}. 
\end{equation}

\subsection{Estimating $\tau$ and $\rho$}
Assume for the sake of parsimony of the model that we have the same $\rho, \tau$ across the challenge bases. The joint probability model can be extended to include $\tau, \rho$ and their effect on the conditional distribution of $X_{i}, Y_{i}$. The extended model is 
\begin{equation} \label{eq: joint probability model extended}
p_{\delta}(\epsilon, s, \tau, \rho, \alpha_{1:n}, \beta_{1:n}, X_{1:n}, Y_{1:n}) = p(\epsilon)  p(s) p(\tau, \rho) \prod_{i = 1}^{n} p_{\delta}(\alpha_{i}, \beta_{i} | \epsilon, s) g_{\tau, \rho}(X_{i}, Y_{i} | \alpha_{i}, \beta_{i}),
\end{equation}
where $p(\tau, \rho)$ is given in \eqref{eq: prior for tau and rho} and $g_{\tau, \rho}(X_{i}, Y_{i} | \alpha_{i}, \beta_{i}) $ is indicated by \eqref{eq: bivariate distribution-2}. Alg.\ \ref{alg: MCMC-DP-Est extended} is an extension of \texttt{MCMC-DP-Est} that draws samples for $(\epsilon, s, \tau, \rho)$ from the posterior distribution that is proportional to \eqref{eq: joint probability model extended}. Finally, the correctness of this extension can also be established similarly.

\begin{algorithm}[H]
\caption{\texttt{MCMC-DP-Est}: posterior sampling for $(\epsilon, s, \tau, \rho)$}
\label{alg: MCMC-DP-Est extended}

\For{$i = 1:M$}{
Draw the proposal $\epsilon' \sim \log \mathcal{N}(\log \epsilon, \sigma_{q, \epsilon}^{2})$ and $s' \sim \mathcal{N}(s, \sigma_{q,s}^{2})$, $\tau' \sim \mathcal{N}(s, \sigma_{q,\tau}^{2})$, $\rho' \sim \mathcal{N}(\rho, \sigma_{q,\rho}^{2})$.

\For{$j = 1:n$}{
Set $(\alpha_{j}^{(1)}, \beta_{j}^{(1)}) = (\alpha_{j}, \beta_{j})$.

Sample $\alpha_{j}^{(k)}, \beta_{j}^{(k)} \overset{\text{iid}}{\sim} \text{Unif}(0, 1)$ for $k = 2, \ldots, K$.

Calculate the weights
\begin{align*}
w_{j}^{(k)} &= p_{\delta}(\alpha_{j}^{(k)}, \beta_{j}^{(k)} | \epsilon, s) g_{\tau, \rho}(X_{j}, Y_{j} | \alpha_{j}^{(k)}, \beta_{j}^{(k)}), \quad k = 1, \ldots, K\\
w_{j}^{\prime(k)} &=  p_{\delta}(\alpha_{j}^{(k)}, \beta_{j}^{(k)} | \epsilon', s') g_{\tau', \rho'}(X_{j}, Y_{j} |  \alpha_{j}^{(k)}, \beta_{j}^{(k)}) \quad k = 1, \ldots, K
\end{align*}
}

Acceptance probability: 
\[
A = \min \left\{1, \frac{p(s') p(\epsilon')  \epsilon'}{p(s) p(\epsilon) \epsilon } \frac{ p(\tau', \rho')}{p(\tau, \rho)}\prod_{j = 1}^{n}\frac{ \sum_{k = 1}^{K} w_{j}^{\prime(k)}}{\sum_{k = 1}^{K} w_{j}^{(k)}} \right\}.
\]

\textbf{Accept/Reject}: Draw $u \sim \text{Unif}(0, 1)$.

\uIf{$u \leq A$}{
Set $\epsilon = \epsilon', s = s', \tau = \tau', \rho = \rho'$, and $\bar{w}_{1:n}^{(1:K)} = w_{1:n}^{\prime (1:K)}$.
}\Else{
Keep $\epsilon, s, \tau, \rho$ and set $\bar{w}_{1:n}^{(1:K)} = w_{1:n}^{(1:K)}$.
}
\For{$j = 1, \ldots, n$}{
Sample $k \in \{1, \ldots, K\}$ w.p.\ $\propto \bar{w}_{j}^{(k)}$ and set $(\alpha_{j}, \beta_{j}) = (\alpha_{j}^{(k)}, \beta_{j}^{(k)})$.
}
Store $\epsilon^{(i)} = \epsilon, s^{(i)} = s, \tau^{(i)} = \tau, \rho^{(i)} = \rho$.
}

\end{algorithm}

\subsection{Additional details and results for the experiments in Section 4.2}

We used random walk proposals for $\log \epsilon, s, \tau, \rho$ with proposal variances $\sigma_{q, \epsilon}^{2} = 10^{-2}$, $\sigma_{q, s}^{2} = 10^{-4}$, $\sigma_{q, \tau}^{2} = 10^{-6}$, $\sigma_{q, \rho}^{2} = 10^{-6}$. The hyperparameters of the priors are $\epsilon \sim \mathcal{N}_{[0, \infty)}(0, 10)$, $s \sim \text{Unif}(0, 1)$, and the hyperparameters in \eqref{eq: prior for tau and rho} are taken as $\sigma_{\tau}^{2} = 10^{-4}$ and $\sigma_{\rho}^{2} = 10^{-2}$. 

The main paper shows the 2D histograms of $\epsilon, s$ and sample ACF obtained from the samples from \texttt{MCMC-DP-Est} for six versions of $\mathcal{A}$. In addition, we report the trace plots for those samples in Figure \ref{fig: trace plots}.

\begin{figure}[h!]
\centerline{
\includegraphics[scale = 0.6]{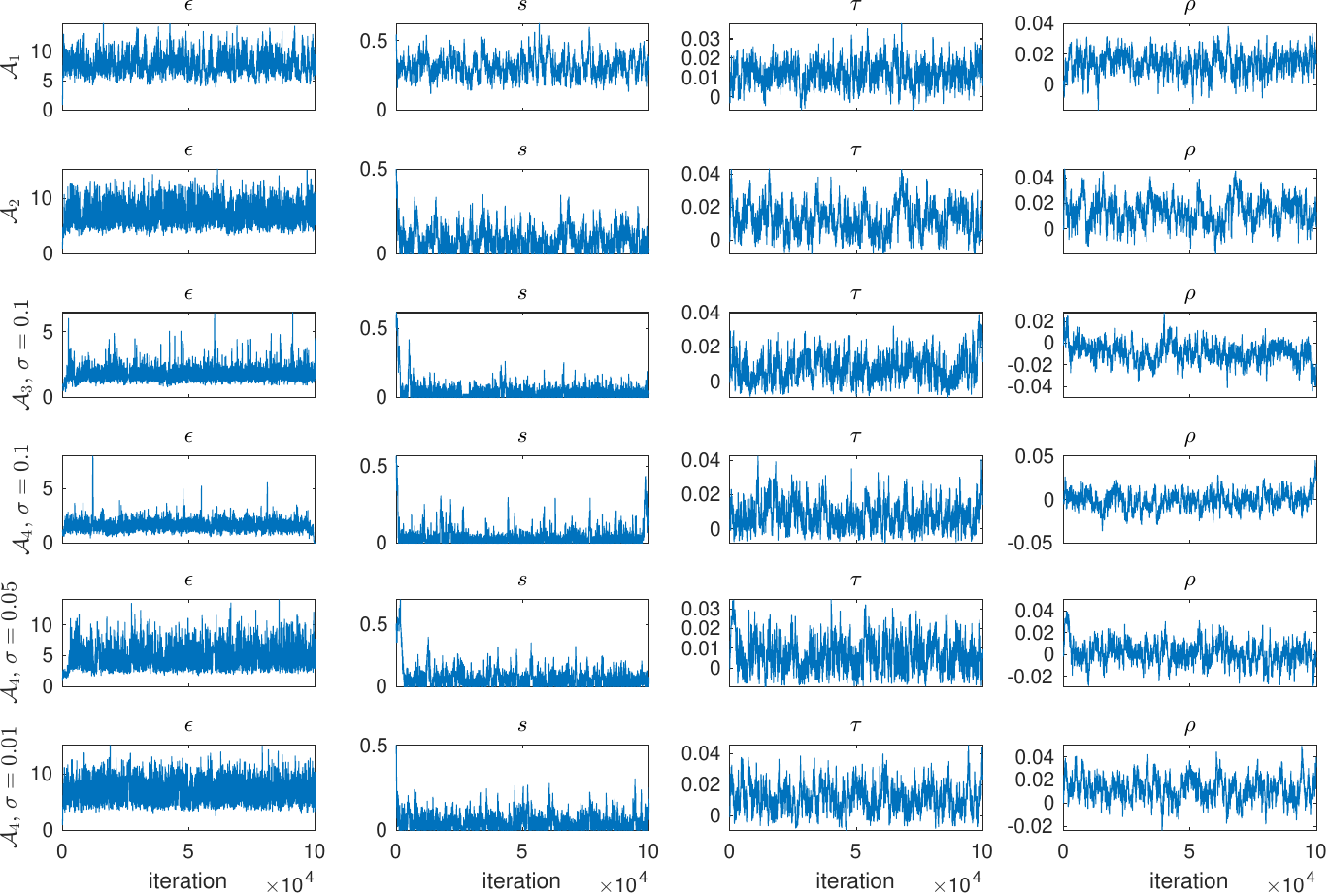}}
\caption{Trace plots for all the four parameters $\epsilon$, $s$, $\tau$, $\rho$.}
\label{fig: trace plots}
\end{figure}

\section{Modeling Dependent MIAs} \label{appndx: Modeling Dependent MIAs}
Dependent MIAs can occur, for example, when their challenge bases share a common $z$ (or $D$), or they share a common challenge base but use different decision rules. 

Assume there are $n$ groups of dependent MIAs. If MIAs in the same group are expected to have close performance, we can model their joint distribution as follows.  For group $i$ of size $m_{i}$, let
\[
(\alpha_{i}, \beta_{i}) | \epsilon, \delta, s \overset{i.i.d.}{\sim}
\text{Uniform}(\mathcal{R}_{s}(\epsilon, \delta)), \quad i = 1, \ldots, m_{i}
\]
denote the average test performance of group $i$. Then, the error performances of the MIAs in the $i$'th group can be modeled as
\[
\alpha_{ij}, \beta_{ij} | (\alpha_{i}, \beta_{i}), \epsilon, \delta \overset{iid}{\sim} \mathcal{P}((\alpha_{i}, \beta_{i}), \tau), \quad j = 1, \ldots, m_{i},
\]
where $\mathcal{P}((\alpha_{i}, \beta_{i}), \tau)$ is a bivariate probability distribution that is symmetric around $(\alpha_{i}, \beta_{i})$ and has variance $\tau$. Examples include bivariate normal distribution and bivariate uniform distribution.

Given $\alpha_{ij}, \beta_{ij}$, the error counts $(X_{ij}, Y_{ij})$ are distributed as before, and they are conditionally independent, i.e\,
\[
(X_{ij}, Y_{ij})| \alpha_{i}, \alpha_{i1}, \ldots, \alpha_{in_{i}}, \epsilon, \delta, s \sim g(\cdot | \alpha_{ij}, \beta_{ij}), 
\]
independently for $j = 1, \ldots, m_{i}$. The conditional distribution $g(\cdot | \alpha_{ij}, \beta_{ij})$ only depends on how the $i, j$'th MIA's performance is measured. In particular, it does not depend on the dependency among the true performances of the MIAs.

The resulting hierarchical model is given in Figure \ref{fig: DAG extended model}. A precise Bayesian approach would target the posterior distribution of $\epsilon$, $s$, $\alpha_{i}$'s and $\alpha_{ij}$'s, given $\{ X_{ij}, Y_{ij}: j = 1, \ldots, n_{i}; i = 1, \ldots, m \}$, which can be sampled from by various MCMC algorithms in the literature. 

Alternatively, one may compromise theoretical precision for ease in numerical computation and instead propose to target the posterior distribution of $\epsilon$ given 
\[
X_{i} = \sum_{j = 1}^{J_{i}} X_{ij}, \quad Y_{i} = \sum_{j = 1}^{J_{i}} Y_{ij}, \quad i = 1, \ldots, n.
\]
and use a normal approximation for the joint distribution of $X_{i}, Y_{i}$ via moment matching.
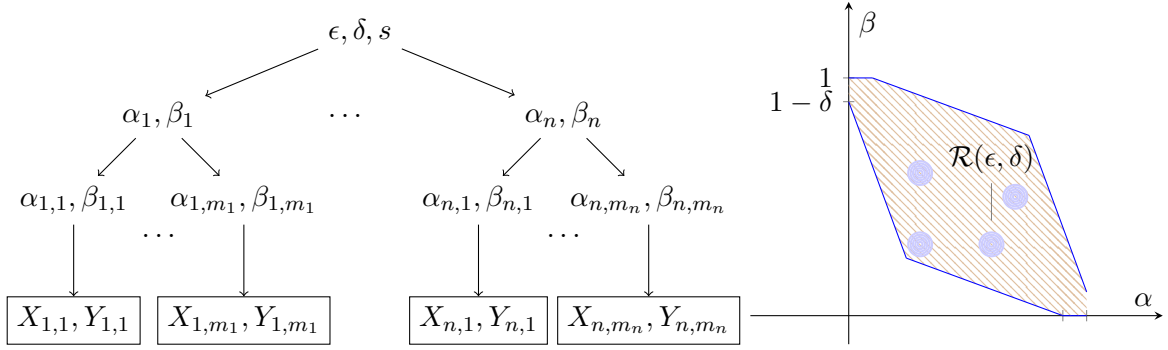
\begin{figure}[h!]
\centerline{
\begin{tikzpicture}[
    ->,
    %>=stealth',
    auto,
    node distance=1.6cm, minimum size=0cm,
    main node/.style={}
]
% Define the nodes
\node[main node] (DP) {$\epsilon, \delta, s$};
\node[main node] (T1) [below left=0.5cm and 1.5cm of DP] {$\alpha_{1}, \beta_{1}$};
\node[main node] (T11) [below left of=T1] {$\alpha_{1, 1}, \beta_{1, 1}$};
\node[] (T1b) [below of=T1] {$\ldots$};
\node[main node] (T1J) [below right of=T1] {$\alpha_{1, m_{1}}, \beta_{1, m_{1}}$};
\node[draw, rectangle] (X11) [below of =T11] {$X_{1, 1}, Y_{1, 1}$};
\node[draw, rectangle] (X1J) [below of =T1J] {$X_{1, m_{1}}, Y_{1, m_{1}}$};
\node[] (Tn) [right=1.5cm of T1] {$\ldots$};
\node[main node] (Tn) [below right=0.5cm and 1.5cm of DP] {$\alpha_{n}, \beta_{n}$};
\node[main node] (Tn1) [below left of=Tn] {$\alpha_{n, 1}, \beta_{n, 1}$};
\node[] (Tnb) [below of=Tn] {$\ldots$};
\node[main node] (TnJ) [below right of=Tn] {$\alpha_{n, m_{n}}, \beta_{n, m_{n}}$};
\node[draw, rectangle] (Xn1) [below of =Tn1] {$X_{n, 1}, Y_{n, 1}$};
\node[draw, rectangle](XnJ) [below of =TnJ] {$X_{n, m_{n}}, Y_{n, m_{n}}$};

% Draw the edges
\path[every node]
    (DP) edge[] node[above] {} (T1)
    (DP) edge[] node[above] {} (Tn)
    (T1) edge[] node[above] {} (T11)
        edge[] node[above] {} (T1J)
    (T11) edge[] node[above] {} (X11)
    (T1J) edge[] node[above] {} (X1J)
    (Tn) edge[] node[above] {} (Tn1)
        edge[] node[above] {} (TnJ)
    (Tn1) edge[] node[above] {} (Xn1)
    (TnJ) edge[] node[above] {} (XnJ);
\end{tikzpicture}
\hfill
\begin{tikzpicture}
\pgfmathsetmacro{\dDP}{0.1}  % Set a to 2
\pgfmathsetmacro{\eDP}{1}  % Set a to 2
\begin{axis}[axis lines=middle,
            axis equal, % to keep the aspect ratio correct
            xlabel=$\alpha$,
            ylabel=$\beta$,
            enlargelimits,
            ytick={0, 0.9, 1},
            yticklabels={0, $1-\delta$, 1},
            xtick={0, 0.9, 1},
            xticklabels={0,$1-\delta$, 1}
            xmin=-0, xmax=1.2, ymin=0, ymax=1.2, scale = 0.8]
\addplot[name path=B1,blue,domain={0:(1 - \dDP)/(1 + e^\eDP)}] {(1-\dDP)-e^(\eDP)*(x)};
\addplot[name path=B2,blue,domain={((1 - \dDP)/(1 + e^\eDP)):(1 - \dDP)}] {(1 - \dDP)/(1+e^\eDP)-e^(-\eDP)*(x-(1 - \dDP)/(1+e^\eDP))};
\addplot[name path=B0,blue,domain={(1 - \dDP):1}] {0};
\addplot[name path=A0,blue,domain={0: \dDP}] {1};
\addplot[name path=A1,blue,domain={\dDP: (1 + \dDP*e^-\eDP)/(1+e^-\eDP)}] {1 - (x-\dDP)*e^(-\eDP)};
\addplot[name path=A2,blue,domain={(1 + \dDP*e^(-\eDP))/(1+e^(-\eDP)): 1}] {(1 + \dDP*e^(-\eDP))/(1+e^(-\eDP)) - e^(\eDP)*(x-(1 + \dDP*e^(-\eDP))/(1+e^(-\eDP)))};
\addplot[pattern=north west lines, pattern color=brown!50]fill between[of=B1 and A0, soft clip={domain=0:1}];
\addplot[pattern=north west lines, pattern color=brown!50]fill between[of=B1 and A1, soft clip={domain=0:1}];
\addplot[pattern=north west lines, pattern color=brown!50]fill between[of=B2 and A2, soft clip={domain=0:1}];
\addplot[pattern=north west lines, pattern color=brown!50]fill between[of=B0 and A2, soft clip={domain=0:1}];
\node[coordinate,pin=90:{$\mathcal{R}(\epsilon, \delta)$}] at (axis cs:0.6,0.4){};
\foreach \cx/\cy in {0.3/0.3, 0.3/0.6, 0.6/0.3, 0.7/0.5} {
    % Inner loop: Draw concentric circles at each center
    \foreach \r/\i in {0.01/1, 0.02/2, 0.03/3, 0.04/4, 0.05/5} {
        \addplot[domain=0:360, samples=100, thick, color=blue!20]
                ({\cx + \r*cos(x)}, {\cy + \r*sin(x)});
    };
};
\end{axis}
\end{tikzpicture}
}
\caption{Left: DAG for the dependence structure among the variables involved in the model with dependent tests. Right: An illustration with $n = 4$ groups of tests with $s = 0$. The centers of the circles represent $\alpha_{i}, \beta_{i}$ and the circles themselves represent the conditional distributions of $(\alpha_{ij}, \beta_{ij})$ given $\alpha_{i}, \beta_{i}$}
\label{fig: DAG extended model}
\end{figure}

\section{Most powerful test for comparing two normal distributions with a single observation} \label{appndx: Most powerful test for comparing two normal distributions with a single observation}
According to the Neyman-Pearson Lemma, the rejection rule most powerful test for comparing two univariate normal distributions
\[
H_{0}: X \sim \mathcal{N}(\mu_{0}, \sigma_{0}^{2}), \quad H_{1}: X \sim \mathcal{N}(\mu_{1}, \sigma_{1}^{2})
\]
with a single observation has the form
\[
\text{LR} = \frac{\frac{1}{\sqrt{2 \pi \sigma_{0}^{2}} } \exp\{-\frac{1}{2 \sigma_{0}^{2}} (x - \mu_{0})^{2} \}}{\frac{1}{\sqrt{2 \pi \sigma_{1}^{2}} } \exp\{-\frac{1}{2 \sigma_{1}^{2}} (x - \mu_{1})^{2} \}} \leq c,
\]
or 
\[
\left( \frac{1}{\sigma_{1}^{2}} - \frac{1}{\sigma_{0}^{2}} \right) x^{2} + 2 \left( \frac{\mu_{0}}{\sigma_{0}^{2}} - \frac{\mu_{1}}{\sigma_{1}^{2}} \right) x \leq c'.
\]
Let $R = \left(\frac{\mu_{0}}{\sigma_{0}^{2}} - \frac{\mu_{1}}{\sigma_{1}^{2}}\right) / \left( \frac{1}{\sigma_{1}^{2}} - \frac{1}{\sigma_{0}^{2}} \right)$. Let $Z \sim \mathcal{N}(0, 1)$, Let $Y =  (Z + \delta)^{2}$ with  $\delta = \frac{\mu_{0}}{\sigma_{0}} + \frac{R}{\sigma_{0} }$. Then, $Y$ has a noncentral chi-square distribution $\chi^{2}_{1, \delta^{2}}$.

We inspect three cases in terms of the sign of $\sigma_{0}^{2} - \sigma_{1}^{2}$:
\begin{enumerate}
\item Assume $\sigma_{0}^{2} > \sigma_{1}^{2}$. Then the rejection rule becomes $\left( x + R \right)^{2} \leq c_{\alpha}$, To determine $c_{\alpha}$, write
\[
P\left[\left( X + R \right)^{2} \leq c_{\alpha} \right]  = P\left[\left(  Z + \frac{\mu_{0}}{\sigma_{0}} + \frac{R }{\sigma_{0}} \right)^{2} \leq \frac{c_{\alpha} }{\sigma_{0}^{2}}\right] = P\left(Y^{2} \leq \frac{c_{\alpha}}{\sigma_{0}^{2}}\right) = \alpha.
\]
Therefore, $c_{\alpha} = \sigma_{0}^{2} \cdot F^{-1}_{\chi^{2}_{1, \delta}}(\alpha)$.

\item Assume $\sigma_{0}^{2} < \sigma_{1}^{2}$. The rule reduces to $\left( x +R \right)^{2} \geq c_{\alpha}$. To determine $c_{\alpha}$, write
\[
P\left[\left( X + R \right)^{2} \geq c_{\alpha} \right]  = P\left[\left(  Z + \frac{\mu_{0}}{\sigma_{0}} + \frac{R}{\sigma_{0} } \right)^{2} \geq \frac{c_{\alpha} }{\sigma_{0}^{2}}\right] = P\left( Y^{2} \geq \frac{c_{\alpha} }{\sigma_{0}^{2}} \right) = \alpha.
\]
Therefore, $c_{\alpha} = \sigma_{0}^{2} \cdot F^{-1}_{\chi^{2}_{1, \delta}}(1 - \alpha)$.

\item Assume $\sigma_{0}^{2} = \sigma_{1}^{2}$. Then, the decision rule is $\text{sign}(\mu_{1} -\mu_{0}) (X - \mu_{0})/\sigma > F^{-1}_{\mathcal{N}(0, 1)}(1-\alpha)$.
\end{enumerate}

\end{document}